%% file: main.tex
\documentclass{iopjournal}

\usepackage{graphicx}
\usepackage{subcaption}

\usepackage{amsmath,amsfonts}
\newtheorem{theorem}{Theorem}
\newtheorem{corollary}{Corollary}
\newenvironment{proof}{\paragraph{\textit{Proof:}}}{}

\newtheorem{remark}{Remark}
\usepackage{comment}
\usepackage[utf8]{inputenc} 
\usepackage[T1]{fontenc}    
\usepackage{hyperref}       
\usepackage{url}            
\usepackage{booktabs}       
\usepackage{amsfonts}       
\usepackage{nicefrac}       
\usepackage{microtype}      
\usepackage{algorithm} 
\usepackage{lipsum}		
\usepackage{graphicx}
\usepackage{amsmath}
\usepackage{algorithm} 
\usepackage{algpseudocode} 
\usepackage{bbm}
\usepackage{amsfonts}
\usepackage{natbib}
\usepackage{algpseudocode} 
\usepackage{doi}
\usepackage{tikz}
\usetikzlibrary{decorations.pathreplacing,arrows.meta}

\input{Macros.tex}

\begin{document}

\articletype{Paper} 

\title{An optimal control approach  for neural network  architecture adaptation with a posteriori error estimation}

\author{C G Krishnanunni$^1$, Thomas Scott$^1$,  Tan Bui-Thanh$^{1,*}$}

\affil{$^1$Department of Aerospace Engineering \& Engineering Mechanics, UT Austin}

\affil{$^*$Oden Institute for Computational Engineering and Sciences,  UT Austin.}

\email{krishnanunni@utexas.edu}


\begin{abstract}
This work presents a novel approach for adapting neural network architecture along the depth based on a posteriori error estimation. By formulating neural network training as a continuous-time optimal control problem, we derive rigorous error estimates that quantify how approximation error distributes across network layers. This error decomposition enables a principled depth adaptation strategy: new layers are inserted at locations of maximum estimated error, allowing the network to efficiently capture complex, nonlinear variations in the underlying problem. Our framework introduces a novel network architecture that treats weights and biases as piecewise linear functions varying across layers, with the error estimator bounding the discrepancy between this discrete representation and the true continuous optimal control solution. The approach leverages dual weighted residual methodology from finite element analysis to derive computable upper bounds on the functional error. A key theoretical contribution is the derivation of explicit error bounds that decompose the total approximation error into interval-wise contributions, providing a rigorous basis for targeted architecture refinement. We demonstrate the effectiveness of our method on scientific datasets, including learning the observable-to-parameter map for the Navier-Stokes equation. Numerical results reveal that our approach consistently outperforms existing architecture adaptation methods in terms of generalization performance.
\end{abstract}

\section{Introduction}

Depth of a neural network plays a key role in the observed empirical success of deep learning. Stacking layers lets a network build successively more abstract features out of raw input data \cite{hinton2007learning,zeiler2014visualizing}.  Several of the architectures responsible for the largest jumps in benchmark performance over the past decade differ from their predecessors mainly in how many layers they contain \cite{simonyan2014very,szegedy2015going}. A key question is how many layers does a given task actually require, and how should the available parameters be distributed among them? In practice this question is still settled almost entirely by trial and error or a metaheuristic optimization-based search procedure \cite{suganuma2017genetic,liu2021survey}.  A network with too little depth and small width may be unable to represent the target function at all, whereas an unnecessarily deep network wastes compute and is more prone to overfitting. A more systematic, less wasteful approach to setting network depth therefore has clear practical value.

Two broad families of methods have been proposed to automate this choice. The first treats architecture selection as a search problem.  A population or sequence of candidate networks is proposed, evaluated using evolutionary computation \cite{suganuma2017genetic,liu2021survey}, reinforcement learning \cite{zoph2016neural}, or randomized search \cite{li2020random}, and the best candidate is retained. These neural architecture search (NAS) methods can produce excellent architectures, but doing so typically requires training large numbers of candidates which makes the overall procedure costly, particularly when a single training run is itself expensive. The second family avoids searching a discrete space of architectures altogether. Instead, a single small network is trained and then grown, with new parameters added as training proceeds rather than fixed in advance \cite{wynne1991node,evci2022gradmax,wu2020firefly,krishnanunni2025topological,krishnanunni2025adaptive}.

Within this second family, the bulk of prior work targets towards growing the width of a network rather than its depth. Wynne-Jones \cite{wynne1991node} grew networks by splitting neurons identified through a principal component analysis. Firefly algorithm by Wu et al. \cite{wu2020firefly} generates candidate neurons that either split existing
neurons with noise or are completely new and selects those with the highest gradient norm.  A related line of work \cite{wu2019splitting} frames neuron splitting as a steepest descent step over probability measures in Wasserstein space. Evci et al. \cite{evci2022gradmax} instead select new neurons directly from gradient information, choosing those most likely to reduce the loss once activated. For growing an architecture along the depth,  Bengio et al.  \cite{bengio2006greedy} proposed a greedy layerwise unsupervised learning algorithm where the initial layers of a network are supposed to represent more abstract concepts that explain the input observation, whereas subsequent layers extract low-level features. This idea was later adapted to supervised learning by Hettinger et al.'s Forward Thinking algorithm \cite{hettinger2017forward}.  Net2Net \cite{chen2015net2net} takes a different route, inserting a layer initialized to be function preserving transformation such that the deeper network starts at exactly the loss value of its shallower predecessor before any retraining occurs. 
 Sensli algorithm \cite{kreis2023sensli} builds on Net2Net by also choosing where the new layer goes, based on a sensitivity criterion, though its initialization of the inserted layer remains the same regardless of the training data or the chosen insertion site. Recently, we have proposed two depth growing strategies of our own. The first is a manifold regularized greedy layerwise training approach for adapting a neural architecture along its depth \cite{krishnanunni2025adaptive}. The second computes a closed-form topological derivative borrowed from topology optimization to decide both the location and initialization of an added new layer \cite{krishnanunni2025topological}.

What is missing from most of this literature is a quantitative link between where a new layer is added and how the network's approximation error decomposes across its layers. Such a measure becomes available once a residual network is viewed not as a stack of discrete layers but as a discretization of a continuous-time dynamical system. Under this viewpoint, training amounts to solving a discrete-time approximation of an optimal control problem in which the controls are the layer-wise weights and biases \cite{li2018optimal,chen2018neural}. Reframing network training this way opens the door to importing a substantial body of numerical analysis machinery built for exactly this kind of problem (architecture adaptation). In particular, the dual weighted residual (DWR) approach \cite{becker2001optimal,kraft2010dual} was developed in the finite element community to answer a structurally identical question: given a discretized solution to a PDE constrained optimization problem, where in the computational domain should the mesh be refined to most reduce the error in a quantity of interest \cite{becker2001optimal,kraft2010dual} ? It answers this question by writing the error as a sum of local, computable residuals attached to individual mesh elements.

This paper develops a procedure for growing neural network architecture along the depth, motivated by the dual weighted residual (DWR) approach. Rather than treating the standard forward Euler discretization used in practice as the starting point, we represent a network's weights and biases as piecewise linear functions of a continuous depth variable and discretize the underlying optimal control problem with a finite element method (section \ref{t_sp}). Theorem \ref{theo_one} gives an exact expression for the gap between the loss achieved by such a discretized network and the loss that would be achieved if its parameters varied continuously with depth. Corollary \ref{cor_one} converts this expression into a fully computable upper bound that splits additively across the layer-to-layer intervals of the network (section \ref{er_est}), directly identifying which interval is contributing the most error. This interval-wise bound drives the adaptation procedure in Algorithm \ref{Algo_full}. A network of fixed depth is trained, the per-interval error is estimated, a new layer is inserted at the interval with the largest estimated error and initialized by interpolating its two neighbors, and the cycle repeats as long as a held-out validation loss keeps improving. In this way, the three design choices that any depth growing scheme must make, namely i) where to insert a layer, ii) when to stop inserting layers, and iii) how to initialize each new layer, are all resolved. Section \ref{num_d} reports numerical results on a synthetic two-dimensional regression problem and on a severely ill-posed inverse problem for the Navier-Stokes equation, where the resulting networks generalize better than those produced by Net2Net \cite{chen2015net2net}, Forward Thinking \cite{hettinger2017forward}, and random layer insertion.

\section{Mathematical framework}

In this work, all matrices are denoted by boldface capital letters, while all vectors are denoted by boldface lowercase letters. Consider a regression/classification task where one is provided with $S$ training data points, input data dimension $n_0$, and label dimension $n_{T+1}$. Let the inputs $\bx_s \in \mathbb{R}^{n_0}$  
for $s \in \{1,2,...S\}$ be organized
row-wise into a matrix $\mathbf{X} \in \mathbb{R}^{S\times {n_0}}$ and the corresponding true labels be denoted as $\bc_s\in \mathbb{R}^{n_{T+1}}$ and  stacked row-wise as $\mathbf{C} \in \mathbb{R}^{S\times {n_{T+1}}}$.

\subsection{Neural network training problem}
Let us consider a fully connected neural network (FNN) employing residual connections \cite{he2016deep}. We begin by definining the mesh $t_1<t_2<....<t_T$, with step size $h_{n-1} =t_n-t_{n-1}$ and the corresponding intervals be denoted as $I_{n-1}=(t_{n-1},t_n)$, $n=2,3,\dots, T$. 
Given  the inputs $\bx_s$  
and the corresponding true labels  $\bc_s$, we consider the following neural network training problem:

Find $\LRp{\bW^h_0,\dots,\bW^h_{T},\bb^h_0,\dots \bb^h_{T}}$ that solves the following optimal control problem:
\begin{equation}
 \begin{aligned}
   \text{minimize} &\ \ \frac{1}{S}\sum_{s=1}^S\mathcal{J}\LRp{  \bx_{s;T+1}^h} \\ 
    \text{subject  to} \ \  
     &\ \  \bx^h_{s;1}=\sigma_1\LRp{\bW^h_0\bx^h_{s;0}+\bb^h_0},\quad \bx_{s;0}^h=\bx_s,\\
    & \ \ \bx_{s;t+1}^h=\bx_{s;t}^h+h_{t} \ \sigma\LRp{\bW^h_{t}\bx^h_{s;t}+\bb^h_{t}}, \ \ \forall t\in 1,\dots T-1,\\
    &  \bx^h_{s;T+1}=\sigma_{T+1}\LRp{\bW^h_{T}\bx^h_{s;T}+\bb^h_{T}}
    \end{aligned}   
\label{res_hid_dis}
\end{equation}
where $ s\in \{1,\dots S\}$, $\bx_{s;t}$ denotes the state at the $t-$th layer of the network corresponding to $s$-th training sample,  $\sJ$ denotes the loss function, $\sigma_1(.),\ \sigma(.), \ \sigma_{T+1}(.)$ are  activation functions acting component-wise on the inputs, $\bW^h_t\in \real^{n_{t+1}\times n_{t}}$, $\bb^h_t\in \real^{n_{t+1}}$ denote the weights and biases, $n_t$ denotes  the number of neurons in the $t-$th layer,  $T$ is the total number of hidden layers in the network.

\subsection{Neural network training: Continuous problem}

Consider \eqref{res_hid_dis} where $h_t\rightarrow 0$. 
In the continuum, therefore
one has the following continuous optimal control problem \cite{benning2019deep} for training a neural network:

Find $\LRp{\bW_0,\bW_{T},\bb_0, \bb_{T}}$,  $\bW(t)\in \sUw$, $\bb(t)\in \sUb$ that solves the following optimal control problem:
\begin{equation}
 \begin{aligned}
  \text{minimize} &\ \ \frac{1}{S}\sum_{s=1}^S\mathcal{J}\LRp{  \bx_{s;T+1}} \\   
    \text{subject  to} \ \  
    &\ \  \bx_s(t_1)=\sigma_1\LRp{\bW_0\bx_{s;0}+\bb_0},\quad \bx_{s;0}=\bx_s,\\
    & \ \ \dot{\bx}_s(t)=\sigma\LRp{\bW(t)\bx_s(t)+\bb(t)}, \ \ t_1\leq t \leq t_T,\\
     &  \bx_{s;T+1}=\sigma_{T+1}\LRp{\bW_{T}\bx_{s}(t_T)+\bb_{T}},
    \end{aligned}   
\label{stat_final}
\end{equation}
where $ \sUw,\ \sUb$ are appropriate spaces that one assumes, and we also assumed that $\xb_s(t)$ is continuous and differentiable everywhere.
\paragraph{Goal}
Our goal in this work is to derive an expression (upper bound) for the following quantity using only the solution from \eqref{res_hid_dis}:
\begin{equation}
\snor{\frac{1}{S}\sum_{s=1}^S\mathcal{J}\LRp{  \bx_{s;T+1}^h}- \frac{1}{S}\sum_{s=1}^S\mathcal{J}\LRp{  \bx_{s;T+1}}}.
\label{goal}
\end{equation}
If one can characterize how the error in \eqref{goal} is distributed across the intervals $I_{n-1}$ in \eqref{res_hid_dis}, then a mesh refinement strategy can be devised. In particular, inserting a new hidden layer at the location of maximal error leads naturally to an architecture (depth) adaptation framework for neural networks.

\begin{remark}
To derive an estimate for \eqref{goal}, we employ the dual weighted residual approach for a posteriori error estimation developed in \cite{kraft2010dual}. The dual weighted residual framework is variational in nature and relies on discretizing \eqref{stat_final} using a finite element method.

However, the discrete optimal control problem \eqref{res_hid_dis} is obtained by applying a forward Euler scheme to discretize the ordinary differential equation (ODE) in \eqref{stat_final}. Consequently, it is necessary to first reformulate and rederive \eqref{res_hid_dis} by discretizing the continuous optimal control problem \eqref{stat_final} using a finite element method.
\end{remark}
To this end, we restrict our discussion to the case where the input weights and biases, namely $\bW_0$ and $\bb_0$, are fixed and given. We also assume that the output weights and biases, $\bW_T$ and $\bb_T$, are fixed and given.
Under these assumptions, the optimal control problem \eqref{stat_final} can be rewritten as follows:

Find   $\bW(t)\in \sUw$, $\bb(t)\in \sUb$ that solves:
\begin{equation}
 \begin{aligned}
  \text{minimize} & \ \ \tilde{\mathcal{J}}\LRp{\bx_1(t_T),\ \bx_2(t_T),\dots \bx_S(t_T)}=\frac{1}{S}\sum_{s=1}^S\mathcal{J}\LRp{ \sigma_{T+1}(\bW_{T}\bx_s(t_T)+\bb_{T})}, \\   
    \text{subject  to} \ \  
    & \ \ \dot{\bx}_s(t)=\sigma\LRp{\bW(t)\bx_s(t)+\bb(t)}, \ \ t_1\leq t \leq t_T,
    \end{aligned}   
\label{stat_final_assumption}
\end{equation}
where the initial condition $\bx_s(t_1)$  computed as $\bx_s(t_1)=\sigma_1(\bW_0\bx_{s;0}+\bb_0)$.  

\section{True space and Finite element spaces}
\label{t_sp}
Before discretizing \eqref{stat_final_assumption}, we need to first define 
 the true spaces $\sUw,\ \sUb$ in \eqref{stat_final_assumption}. Let us consider  the mesh $t_1<t_2<....<t_T$, with step size $h_{n-1} =t_n-t_{n-1}$ and the corresponding intervals be denoted as $I_{n-1}=(t_{n-1},t_n)$, $n=2,3,\dots, T$.  The spaces $\sUw,\ \sUb$ are defined as follows:
\begin{equation}
    \begin{aligned}
       & \sUw=\Big \{\bW\in C^0\LRp{[t_1,t_T],\ \ \real^{n_1\times n_1}}:\ \bW|_{I_n}\in C^2\LRp{I_n,\ \real^{n_1\times n_1}},\ n=1,\dots T-1\Big \},\\
      &  \sUb=\Big \{\bb\in C^0\LRp{[t_1,t_T],\ \ \real^{n_1}}:\ \bb|_{I_n}\in C^2\LRp{I_n,\ \real^{n_1}}\ n=1,\dots T-1\Big \},
    \end{aligned}
    \label{true_space}
\end{equation}
where $C^2\LRp{I_n,\ .}$ denotes the space of twice continuously differentiable functions on interval $I_n$,  $C^0\LRp{[t_1,\ t_T],\ .}$ denotes the space of continuous functions on interval $[t_1,\ t_T]$.

Now let us assume the corresponding finite element spaces $\LRp{\sUw}_h$ and $\LRp{\sUb}_h$ as follows:
\begin{equation}
    \begin{aligned}
       & \LRp{\sUw}_h=\{\bW^h\in C^0\LRp{[t_1,t_T],\ \ \real^{n_1\times n_1}}:\ \bW^h|_{I_n}\in P^1\LRp{I_n,\ \real^{n_1\times n_1}}\ n=1,\dots T-1\},\\
       & \LRp{\sUb}_h=\{\bb^h\in C^0\LRp{[t_1,t_T],\ \ \real^{n_1}}:\ \bb^h|_{I_n}\in P^1\LRp{I_n,\ \real^{n_1}}\ n=1,\dots T-1\},
    \end{aligned}
    \label{fem_space}
\end{equation}
where $P^1\LRp{I_n,\ .}$  denotes the space of polynomials of degree at most one on the interval $I_n$.
Note that one has $ \LRp{\sUw}_h\subset  \sUw$ and  $ \LRp{\sUb}_h\subset  \sUb$. 

In particular, with the assumed spaces in \eqref{fem_space}, on each interval $I_{n-1},\ n=2,\dots T$ we have the following representation for the weights and biases:
\begin{equation}
\begin{aligned}
 &\bW^h(t)=\LRp{1-\frac{t-t_{n-1}}{h_{n-1}}}\bW^h(t_{n-1})+\LRp{\frac{t-t_{n-1}}{h_{n-1}}}\bW^h(t_{n}),\quad t_{n-1}\leq t\leq t_n\\
 & \bb^h(t)=\LRp{1-\frac{t-t_{n-1}}{h_{n-1}}}\bb^h(t_{n-1})+\LRp{\frac{t-t_{n-1}}{h_{n-1}}}\bb^h(t_{n}), \quad t_{n-1}\leq t\leq t_n,
    \end{aligned}
    \label{rep}
\end{equation}
where $\bW^h(t_{n})$, $\bb^h(t_{n})$ are the weights and biases prescribed at time $t=t_n$.

\subsection{Deriving the a-posteriori error estimate for neural network}

With the spaces introduced in \eqref{true_space} and \eqref{fem_space}, we will now consider two optimal control problem for neural network training as follows:
\paragraph{True neural network training:}

 Find   $\bW(t)\in \sUw$, $\bb(t)\in \sUb$ such that:
\begin{equation}
 \begin{aligned}
  \text{minimize} & \ \ \tilde{\mathcal{J}}\LRp{\{\bx_s(t_T)\}_{s=1}^S  }=\frac{1}{S}\sum_{s=1}^S\mathcal{J}\LRp{ \sigma_{T+1}(\bW_{T}\bx_s(t_T)+\bb_{T})}, \\   
    \text{subject  to} \ \  
    & \ \ \dot{\bx}_s(t)=\sigma\LRp{\bW(t)\bx_s(t)+\bb(t)}, \ \ t_1\leq t \leq t_T,
    \end{aligned}   
\label{stat_final_true}
\end{equation}
where the notation $\tilde{\mathcal{J}}\LRp{\{\bx_s(t_T)\}_{s=1}^S  }$ means $\tilde{\mathcal{J}}\LRp{\{\bx_s(t_T)\}_{s=1}^S  }=\tilde{\mathcal{J}}\LRp{\bx_1(t_T),\ \bx_2(t_T),\dots \bx_S(t_T)}$.
\paragraph{Coarse neural network training:} Find   $\bW^h(t)\in \LRp{\sUw}_h$, $\bb^h(t)\in \LRp{\sUb}_h$ such that:
\begin{equation}
 \begin{aligned}
  \text{minimize} & \ \ \tilde{\mathcal{J}}\LRp{\{\bx_s^h(t_T)\}_{s=1}^S  }=\frac{1}{S}\sum_{s=1}^S\mathcal{J}\LRp{ \sigma_{T+1}(\bW_{T}\bx_s^h(t_T)+\bb_{T})}, \\   
    \text{subject  to} \ \  
    & \ \ \dot{\bx}^h_s(t)=\sigma(\bW^h(t)\bx^h_s(t)+\bb^h(t)), \ \ t_1\leq t \leq t_T,
    \end{aligned}   
\label{stat_final_fem}
\end{equation}
In \eqref{stat_final_true} and \eqref{stat_final_fem}, we assume that $\bx_s(t),\ \bx_s^h(t)$ lies in the space of continuously differentiable functions on $[t_1,\ t_T]$. Our objective now is to derive an expression (upper bound) for the following quantity using only the solution from \eqref{stat_final_fem}:
\begin{equation}
\snor{\tilde{\mathcal{J}}\LRp{\{\bx_s(t_T)\}_{s=1}^S  }-\tilde{\mathcal{J}}\LRp{\{\bx_s^h(t_T)\}_{s=1}^S  }}\leq \sum_{n=1}^{T-1}\sE_n,
\label{goal_new}
\end{equation}
where $\sE_n$ is the contribution of error in interval $I_n$. In particular, we are interested in understanding how  the error in \eqref{goal_new} decomposes across different intervals $I_n$.
\begin{remark}
The goal of our approach is to characterize the error arising from the piecewise linear representation of weights and biases in the neural network. 
\begin{figure}[h!]
\hspace{-0.25 cm}
\includegraphics[scale=0.44, trim=0 0 0 2cm, clip]{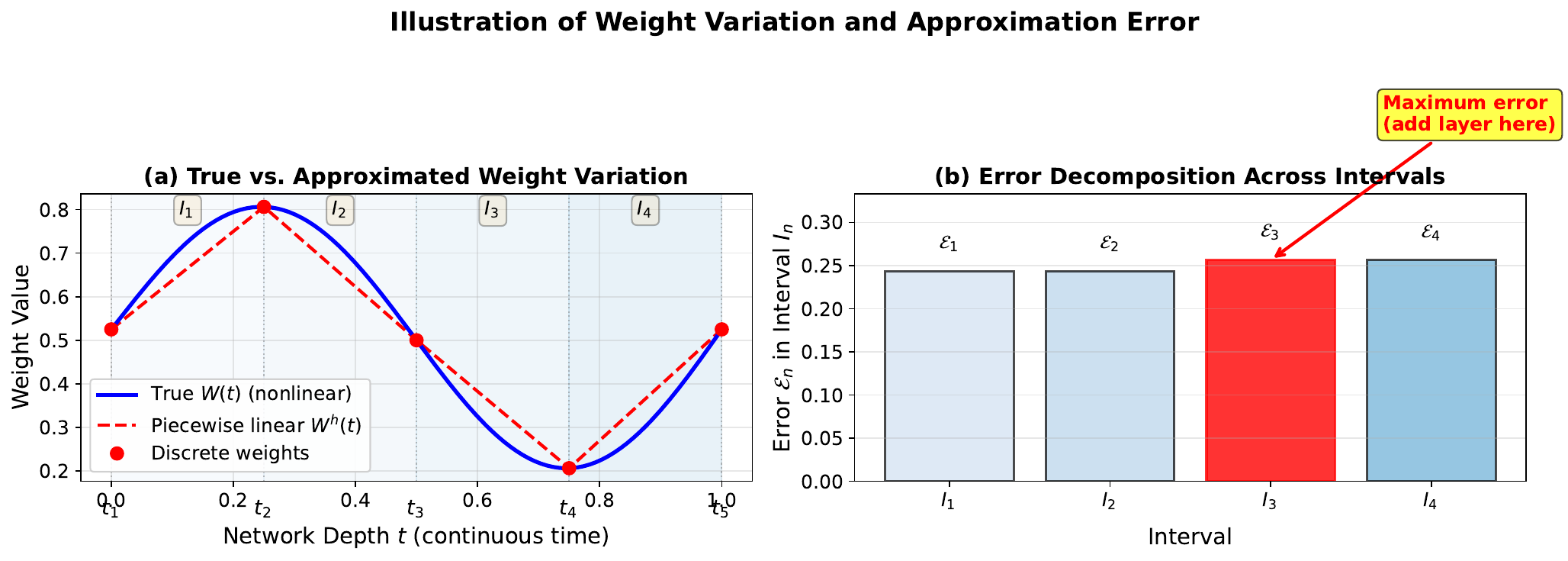}
\caption{Left to Right: The true space $ \sUw,\  \sUb$ and the finite element space $\LRp{\sUw}_h, \LRp{\sUb}_h$ for weights and biases; Decomposition of error \eqref{goal_new} across intervals and the adaptation procedure.}
\label{demo_fig}
\end{figure}
In Figure \ref{demo_fig}, we illustrate this concept: when solving the discrete optimal control problem \eqref{stat_final_fem}, we assume linear interpolation of the weights and biases between layers. However, if the true optimal weights and biases from problem \eqref{stat_final_true} vary nonlinearly, the linear approximation of weights and biases in \eqref{stat_final_fem} introduces error as quantified in \eqref{goal_new}. Our adaptive strategy targets this discrepancy by refining intervals (i.e., adding new layers) where the error is highest, thereby allowing the network to capture more complex, nonlinear parameter variations in those regions.
\end{remark}

\begin{remark}
\label{motiv_remark}
In formulating problem \eqref{stat_final_fem}, we deliberately avoid discretizing the state variables $\bx^h(t)$ using a finite element space. Instead, for the theoretical derivation of the error estimate \eqref{goal_new}, we assume the ODE in \eqref{stat_final_fem} is satisfied exactly. This assumption simplifies the analysis and leads to a cleaner error characterization focused solely on the parameter discretization/representation error.
In practice, once the error estimator is derived, we discretize the ODE in \eqref{stat_final_fem} using a forward Euler scheme with a sufficiently fine mesh, controlled by the discretization parameter K. This practical implementation is detailed in section \ref{discretize_state}. The rationale for avoiding finite element discretization of the states and adjoints is discussed in  remark \ref{no_fem}.
  
\end{remark}

\section{Forming the Lagrangian and analysing the optimality conditions}

To derive an error estimate of the form \eqref{goal_new}, we first examine the first-order optimality conditions for problem \eqref{stat_final_true}. The Lagrangian for \eqref{stat_final_true} can be written as:
\begin{equation}
\begin{aligned}
    \mathcal{L}(\{\bx_s(t)\}_{s=1}^S,\  \bW(t),\ \bb(t),\ \{\bz_s(t)\}_{s=1}^S)=&\tilde{\mathcal{J}}\LRp{\{\bx_s(t_T)\}_{s=1}^S  }\\
    &+\sum_{s=1}^S\int_{t_1}^{t_T} \LRp{\dot{\bx}_s(t)-\sigma(\bW(t)\bx_s(t)+\bb(t)),\ \bz_s(t)}\ dt,
\end{aligned}
    \label{lag_true}
\end{equation}
where $\LRp{.,\ .}$ denotes the Euclidean inner product in $\mathbb{R}^{n_1}$, and $\bz_s(t)$ is the adjoint variable which we assume lies in the space of continuously differentiable functions on $[t_1,\ t_T]$. We also assumed $\bx_s(t)$ lies in the space of continuously differentiable functions on $[t_1,\ t_T]$. Similarly, for \eqref{stat_final_fem} we have:
\begin{equation}
\begin{aligned}
    \mathcal{L}(\{\bx^h_s(t)\}_{s=1}^S,\  \bW^h(t),\ \bb^h(t),\ \{\bz^h_s(t)\}_{s=1}^S)=&\tilde{\mathcal{J}}\LRp{\{\bx_s^h(t_T)\}_{s=1}^S  }\\
    &+\sum_{s=1}^S\int_{t_1}^{t_T} \LRp{\dot{\bx}^h_s(t)-\sigma(\bW^h(t)\bx^h_s(t)+\bb^h(t)),\ \bz^h_s(t)}\ dt,
\end{aligned}
    \label{lag_fem}
\end{equation}
where $\bz^h_s(t)$ is the adjoint variable which we assume lies in the space of continuously differentiable functions on $[t_1,\ t_T]$. We also assumed $\bx^h_s(t)$ lies in the space of continuously differentiable functions on $[t_1,\ t_T]$.
The first-order optimality conditions for  \eqref{lag_fem} and  \eqref{lag_true} are:
\begin{equation}
    \mathcal{L}'\LRp{\{\bx^{h}_s(t)\}_{s=1}^S,\ \bW^{h}(t),\ \bb^{h}(t),\ \{\bz^{h}_s(t)\}_{s=1}^S,\ \be^h}=0,\ \forall \be^h,
    \label{first_order_discrete}
\end{equation}
\begin{equation}
\mathcal{L}'\LRp{\{\bx_s(t)\}_{s=1}^S,\ \bW(t),\ \bb(t),\ \{\bz_s(t)\}_{s=1}^S,\ \be}=0,\ \forall \be,
\label{first_order}
\end{equation}
where
\begin{equation}
\begin{aligned}
 &\mathcal{L}'\LRp{\{\bx^{h}_s(t)\}_{s=1}^S,\ \bW^{h}(t),\ \bb^{h}(t),\ \{\bz^{h}_s(t)\}_{s=1}^S,\ \be^h}\\
 &=\nabla \mathcal{L}\Big(
\{\bx^{h}_s(t)\}_{s=1}^S ,\ 
\bW^{h}(t),\ 
\bb^{h}(t) ,\ 
\{\bz^{h}_s(t)\}_{s=1}^S 
\Big)^{\top}
\begin{bmatrix}
\{\be^h_{\bx_s}\}_{s=1}^S \\
\be_{\bW}^h \\
\be_{\bb}^h \\
\{\be^h_{\bz_s}\}_{s=1}^S
\end{bmatrix}, 
\end{aligned}
\end{equation}
and $\be^h=\LRp{\{\be^h_{\bx_s}\}_{s=1}^S,\ \be^h_{\bW},\ \be^h_{\bb},\ \{\be^h_{\bz_s}\}_{s=1}^S},\ \be=\LRp{\{\be_{\bx_s}\}_{s=1}^S,\ \be_{\bW},\ \be_{\bb},\ \{\be_{\bz_s}\}_{s=1}^S}$ are variations in appropriate spaces, $\nabla \mathcal{L} $ denotes the gradient of the Lagrangian with respect to all variables. We now analyze these first-order optimality conditions in the context of neural network training.

\subsubsection{Forward propagation equations}

Taking the variation of the Lagrangian $\mathcal{L}$ in \eqref{lag_true} with respect to the adjoint variable $\bz_s(t)$, we obtain the weak form:
\begin{equation}
\int_{t_1}^{t_T} \LRp{\dot{\bx}_s(t)-\sigma(\bW(t)\bx_s(t)+\bb(t)),\ \be_{\bz_s}(t)}\ dt = 0, \quad \forall \be_{\bz_s}(t).
\label{weak_forward}
\end{equation}
Since $\bz_s(t)$ lies in the space of continuously differentiable functions and $\be_{\bz_s}(t)$ is an arbitrary variation in this space,  \eqref{weak_forward} must hold for all test functions $\be_{\bz_s}(t)$. By the fundamental lemma of calculus of variations, this implies that the integrand must vanish pointwise, yielding the strong form:
\begin{equation}
\begin{aligned}
    \dot{\bx}_s(t)=\sigma(\bW(t)\bx_s(t)+\bb(t)), \quad t_1\leq t \leq t_T,
\end{aligned}    
\label{forward_prop_cont}
\end{equation}
where $\bx_s(t_1)$ is computed as $\bx_s(t_1)=\sigma_1(\bW_0\bx_{s;0}+\bb_0)$. Similarly, for \eqref{lag_fem} we have:
\begin{equation}
\begin{aligned}
    \dot{\bx}^h_s(t)=\sigma(\bW^h(t)\bx^h_s(t)+\bb^h(t)), \quad t_1\leq t \leq t_T,
\end{aligned}    
\label{forward_prop_fem}
\end{equation}
where $\bx^h_s(t_1)$ is computed as $\bx^h_s(t_1)=\sigma_1(\bW_0\bx_{s;0}+\bb_0)$. 
\begin{remark}[Using the strong form \eqref{forward_prop_fem} instead of a weak form]
\label{no_fem}
We employ the strong form of the state and adjoint equations rather than a weak finite element formulation for the theoretical error analysis. We note that a weak, discontinuous-in-time Petrov–Galerkin formulation can recover the forward Euler method \cite{munoz2019explicit}.

 In this work, however, we choose instead to discretize the strong form of the state and adjoint equations directly with a forward Euler scheme, since this avoids introducing additional finite-element residual terms in the error estimate beyond the parameter discretization/representation error  (see remark \ref{intention}). To control the resulting state/adjoint discretization error, we employ a sufficiently fine sub-mesh, governed by the discretization parameter K (see section \ref{discretize_state}), ensuring that the numerical approximation of the states and adjoints remains accurate.

\end{remark}
\subsubsection{Adjoint equations}

We now consider the variation of the Lagrangian $\mathcal{L}$ in \eqref{lag_true} with respect to the state variable $\bx(t)$. To derive the ODE for $\bz(t)$, we apply integration by parts:
\begin{align*}
 \int_{t_1}^{t_T} &\LRp{\dot{\bx}_s(t)-\sigma(\bW(t)\bx_s(t)+\bb(t)),\ \bz_s(t)}\ dt\\
 &=\int_{t_1}^{t_T}\LRp{\dot{\bx}_s(t),\ \bz_s(t)} \ dt- \int_{t_1}^{t_T}\LRp{\sigma(\bW(t)\bx_s(t)+\bb(t)),\ \bz_s(t)}\ dt\\
&=\LRp{{\bx}_s(t_T),\ \bz_s(t_T)} - \LRp{{\bx}_s(t_1),\ \bz_s(t_1)}- \int_{t_1}^{t_T}\LRp{\bx_s(t),\ \dot{\bz}_s(t)}\ dt\\
&\quad - \int_{t_1}^{t_T}\LRp{\sigma(\bW(t)\bx_s(t)+\bb(t)),\ \bz_s(t)}\ dt.
\end{align*}
Taking the variation of $\mathcal{L}$ with respect to $\bx_s(t)$:
\begin{align*}
 &\LRp{\nabla_s \tilde{\mathcal{J}}\LRp{\{\bx_s(t_T)\}_{s=1}^S  },\ \be_{\bx_s}(t_T)}+\LRp{\be_{\bx_s}(t_T),\ \bz_s(t_T)} - \LRp{\be_{\bx_s}(t_1),\ \bz_s(t_1)}\\
&\quad - \int_{t_1}^{t_T}\LRp{\be_{\bx_s}(t),\ \dot{\bz_s}(t)}\ dt- \int_{t_1}^{t_T}\LRp{\frac{\partial{\sigma(\bW(t)\bx_s(t)+\bb(t))}}{\partial \bx(t)}\be_{\bx_s}(t),\ \bz_s(t)}\ dt=0,\ \ \forall \be_{\bx_s}(t),
\end{align*}
where $\nabla_s \tilde{\mathcal{J}}\LRp{\{\bx_s(t_T)\}_{s=1}^S  }$ denotes the gradient with respect to the variable $\bx_s(t_T)$. Since $\be_{\bx_s}(t_1)=0$ (initial condition is prescribed as $\bx_s(t_1)=\sigma_1(\bW_0\bx_{s;0}+\bb_0)$), the above condition can be simplified as:
\begin{equation}
\begin{aligned}
  &  \LRp{\be_{\bx_s}(t_T),\ \bz_s(t_T)+  \nabla_s \tilde{\mathcal{J}}\LRp{\{\bx_s(t_T)\}_{s=1}^S  }}\\
   & -\int_{t_1}^{t_T}\LRp{\be_{\bx_s}(t),\ \dot{\bz}_s(t)+\LRs{\frac{\partial{\sigma(\bW(t)\bx_s(t)+\bb(t))}}{\partial \bx_s(t)}}^T\bz_s(t)}\ dt =0 ,\ \ \forall \be_{\bx_s}(t).
    \label{weak_form_x}
    \end{aligned}
\end{equation}
 Since $\bx_s(t)$ lies in the space of continuously differentiable functions and $\be_{\bx_s}(t)$ is an arbitrary variation in this space,  \eqref{weak_form_x} must hold for all test functions $\be_{\bx_s}(t)$. This yields the  following  adjoint equations:
\begin{equation}
\begin{aligned}
&\dot{\bz}_s(t)+\LRs{\frac{\partial{\sigma(\bW(t)\bx_s(t)+\bb(t))}}{\partial \bx(t)}}^T\bz_s(t)=0, \\
&\bz_s(t_T)=- \nabla_s \tilde{\mathcal{J}}\LRp{\{\bx_s(t_T)\}_{s=1}^S  }.
\end{aligned}
\label{adjoint_cont}
\end{equation}    
Similarly, for \eqref{lag_fem} we have the following adjoint equations:
\begin{equation}
\begin{aligned}
&\dot{\bz}_s^h(t)+\LRs{\frac{\partial{\sigma(\bW^h(t)\bx_s^h(t)+\bb^h(t))}}{\partial \bx_s^h(t)}}^T\bz_s^h(t)=0, \\
&\bz^h_s(t_T)=-\nabla_s \tilde{\mathcal{J}}\LRp{\{\bx_s^h(t_T)\}_{s=1}^S  }.
\end{aligned}
\label{adjoint_fem}
\end{equation}    
Note that \eqref{adjoint_cont}, and \eqref{adjoint_fem} are solved backwards in time.
\section{Error estimation}
\label{er_est}
With this setting, we are now ready to derive an error estimate for \eqref{goal_new}.
Theorem \ref{theo_one} provides an estimate for the error in \eqref{goal_new}.
\begin{theorem}
\label{theo_one}
Let $\bx_s(t),\  \bW(t),\ \bb(t),\ \bz_s(t)$ are solutions that satisfy the first order optimality condition \eqref{first_order}, and $\bx^h_s(t),\  \bW^h_s(t),\ \bb^h_s(t),\ \bz^h_s(t)$ are solutions that satisfy the first order optimality condition \eqref{first_order_discrete}. In particular, $\bx_s(t),\   \bz_s(t)$ satisfy \eqref{forward_prop_cont} and \eqref{adjoint_cont} respectively; $\bx^h_s(t),\   \bz^h_s(t)$ satisfy \eqref{forward_prop_fem} and \eqref{adjoint_fem} respectively.  Then,
\begin{equation}
  \tilde{\mathcal{J}}\LRp{\{\bx_s(t_T)\}_{s=1}^S}- \tilde{\mathcal{J}}\LRp{\{\bx^{h}_s(t_T)\}_{s=1}^S}=\sum_{n=1}^{T-1}\LRp{\frac{1}{2} \rho_w^n+\frac{1}{2}\rho_b^n}+R,
    \label{main_result}
\end{equation}
where
\[ \rho_{w}^n=-\sum_{s=1}^S\int_{I_n}\LRp{\bW(t)-\hat{\bW}^h(t),\ \LRs{\frac{\partial{\sigma(\bW^{h}(t)\bx^{h}_s(t)+\bb^{h}(t))}}{\partial \bW^{h}(t)}}^T\bz^{h}_s(t)}\ dt \]
\[ \rho_{b}^n=-\sum_{s=1}^S\int_{I_n}\LRp{\bb(t)-\hat{\bb}^h(t),\ \LRs{\frac{\partial{\sigma(\bW^{h}(t)\bx^{h}_s(t)+\bb^{h}(t))}}{\partial \bb^{h}(t)}}^T\bz^{h}_s(t)}\ dt, \]
 Here $\hat{\bb}^h(t)\in \LRp{\sUb}_h$ and $\hat{\bW}^h(t)\in \LRp{\sUw}_h$
are arbitrary, and  $R$ is the remainder (error) term that arises from using the trapezoidal rule to approximate an integral.
\end{theorem}
\begin{proof}
From the definitions of the Lagrangians in \eqref{lag_true} and \eqref{lag_fem}, and noting that both state equations \eqref{forward_prop_cont} and \eqref{forward_prop_fem} are satisfied exactly, we have:
\begin{equation}
\begin{aligned}
   & \tilde{\mathcal{J}}\LRp{\{\bx_s(t_T)\}_{s=1}^S}- \tilde{\mathcal{J}}\LRp{\{\bx^{h}_s(t_T)\}_{s=1}^S}=\\
    &\Lo(\{\bx_s(t)\}_{s=1}^S,\  \bW(t),\ \bb(t),\ \{\bz_s(t)\}_{s=1}^S)-\Lo(\{\bx^{h}_s(t)\}_{s=1}^S,\  \bW^{h}(t),\ \bb^{h}(t),\ \{\bz^{h}_s(t)\}_{s=1}^S).
    \end{aligned}
    \label{beg}
\end{equation}
Applying the integral form of the Taylor expansion to express the Lagrangian at the true solution in terms of the discrete solution we have:
\begin{equation}
\begin{aligned}
    &\Lo(\{\bx_s(t)\}_{s=1}^S,\  \bW(t),\ \bb(t),\ \{\bz_s(t)\}_{s=1}^S)=\Lo(\{\bx^{h}_s(t)\}_{s=1}^S,\  \bW^{h}(t),\ \bb^{h}(t),\ \{\bz^{h}_s(t)\}_{s=1}^S)\\
    & +\int_0^1 {\Lo}'\LRp{\{\bx^{h}_s(t)\}_{s=1}^S+s\{\be_{\bx_s}\}_{s=1}^S,\ \bW^{h}(t)+s\be_{\bW},\ \bb^{h}(t)+s\be_{\bb},\ \{\bz^{h}_s(t)\}_{s=1}^S+s\{\be_{\bz_s}\}_{s=1}^S,\ \be}\ ds,
    \end{aligned}
    \label{fund}
\end{equation}
where
\[ {\Lo}'\LRp{\{\bx^{h}_s(t)\}_{s=1}^S+s\{\be_{\bx_s}\}_{s=1}^S,\ \bW^{h}(t)+s\be_{\bW},\ \bb^{h}(t)+s\be_{\bb},\ \{\bz^{h}_s(t)\}_{s=1}^S+s\{\be_{\bz_s}\}_{s=1}^S,\ \be}= \]
\[  
\nabla \Lo\Big(
\{\bx^{h}_s(t)\}_{s=1}^S + s\{\be_{\bx_s}\}_{s=1}^S,\ 
\bW^{h}(t) + s\be_{\bW},\ 
\bb^{h}(t) + s\be_{\bb},\ 
\{\bz^{h}_s(t)\}_{s=1}^S + s\{\be_{\bz_s}\}_{s=1}^S
\Big)^{\top}
\begin{bmatrix}
\{\be_{\bx_s}\}_{s=1}^S \\
\be_{\bW} \\
\be_{\bb} \\
\{\be_{\bz_s}\}_{s=1}^S,
\end{bmatrix}
\]
and $\be=\LRp{\{\be_{\bx_s}\}_{s=1}^S,\ \be_{\bW},\ \be_{\bb},\ \{\be_{\bz_s}\}_{s=1}^S}$, with $\be_{\bx_s}=\bx_s-\bx^h_s$, $\be_{\bW}=\bW-\bW^h,$ $\be_{\bb}=\bb-\bb^h,$ $\be_{\bz_s}=\bz_s-\bz^h_s$ for each $s\in\{1,\ldots,S\}$.  Adding and subtracting the average of the integrand values at $s=0$ and $s=1$ in \eqref{fund} we have:
\begin{equation}
    \begin{aligned}
      & \Lo(\{\bx_s(t)\}_{s=1}^S,\  \bW(t),\ \bb(t),\ \{\bz_s(t)\}_{s=1}^S)=\Lo(\{\bx^{h}_s(t)\}_{s=1}^S,\  \bW^{h}(t),\ \bb^{h}(t),\ \{\bz^{h}_s(t)\}_{s=1}^S)\\
      & +\int_0^1 {\Lo}'\LRp{\{\bx^{h}_s(t)\}_{s=1}^S+s\{\be_{\bx_s}\}_{s=1}^S,\ \bW^{h}(t)+s\be_{\bW},\ \bb^{h}(t)+s\be_{\bb},\ \{\bz^{h}_s(t)\}_{s=1}^S+s\{\be_{\bz_s}\}_{s=1}^S,\ \be}\ ds\\
      & +\frac{1}{2}  {\Lo}'\LRp{\{\bx^{h}_s(t)\}_{s=1}^S,\ \bW^{h}(t),\ \bb^{h}(t),\ \{\bz^{h}_s(t)\}_{s=1}^S,\ \be}\\
      & -\frac{1}{2}   {\Lo}'\LRp{\{\bx^{h}_s(t)\}_{s=1}^S,\ \bW^{h}(t),\ \bb^{h}(t),\ \{\bz^{h}_s(t)\}_{s=1}^S,\ \be}-\frac{1}{2}  {\Lo}'\LRp{\{\bx_s(t)\}_{s=1}^S,\ \bW(t),\ \bb(t),\ \{\bz_s(t)\}_{s=1}^S,\ \be}
    \end{aligned}
    \label{tra_o}
\end{equation}
where the last term is zero due to first order optimality condition \eqref{first_order}. Note that last two terms in \eqref{tra_o} approximate integral in \eqref{tra_o} by the trapezoidal rule. Therefore we have:
\begin{equation}
    \begin{aligned}
         \Lo(\{\bx_s(t)\}_{s=1}^S,\  \bW(t),\ \bb(t),\ \{\bz_s(t)\}_{s=1}^S)&=\Lo(\{\bx^{h}_s(t)\}_{s=1}^S,\  \bW^{h}(t),\ \bb^{h}(t),\ \{\bz^{h}_s(t)\}_{s=1}^S)\\
         &+\frac{1}{2}  {\Lo}'\LRp{\{\bx^{h}_s(t)\}_{s=1}^S,\ \bW^{h}(t),\ \bb^{h}(t),\ \{\bz^{h}_s(t)\}_{s=1}^S,\ \be}+R,
    \end{aligned}
    \label{use}
\end{equation}
where $R$ is the reminder which is the difference between the integral and the trapezoidal approximation. Using \eqref{use} in \eqref{beg} we have:
\begin{equation}
\begin{aligned}
&\tilde{\mathcal{J}}\LRp{\{\bx_s(t_T)\}_{s=1}^S}- \tilde{\mathcal{J}}\LRp{\{\bx^{h}_s(t_T)\}_{s=1}^S}= \frac{1}{2}  {\Lo}'\LRp{\{\bx^{h}_s(t)\}_{s=1}^S,\ \bW^{h}(t),\ \bb^{h}(t),\ \{\bz^{h}_s(t)\}_{s=1}^S,\ \be} +R\\
=&\frac{1}{2}  {\Lo}'\LRp{\{\bx^{h}_s(t)\}_{s=1}^S,\ \bW^{h}(t),\ \bb^{h}(t),\ \{\bz^{h}_s(t)\}_{s=1}^S,\ \{\bx_s-\bx^h_s\}_{s=1}^S,\ \bW-\bW^h,\ \dots} +R.
\label{one}
\end{aligned}
\end{equation}   
Due to the first order optimality condition \eqref{first_order_discrete} for the discrete problem, we have ${\Lo}'\LRp{\{\bx^{h}_s(t)\}_{s=1}^S,\ \bW^{h}(t),\ \bb^{h}(t),\ \{\bz^{h}_s(t)\}_{s=1}^S,\ \be^h}=0$ for all variations $\be^h$ in the finite element space. This Galerkin orthogonality property allows us to replace $\bW^h$ and $\bb^h$ with arbitrary functions $\hat{\bW}^h\in \LRp{\sUw}_h$ and $\hat{\bb}^h\in \LRp{\sUb}_h$. Therefore, we have:
\begin{equation}
\begin{aligned}
&\tilde{\mathcal{J}}\LRp{\{\bx_s(t_T)\}_{s=1}^S}- \tilde{\mathcal{J}}\LRp{\{\bx^{h}_s(t_T)\}_{s=1}^S}=\\
& \frac{1}{2}  {\Lo}'\LRp{\{\bx^{h}_s(t)\}_{s=1}^S,\ \bW^{h}(t),\ \bb^{h}(t),\ \{\bz^{h}_s(t)\}_{s=1}^S,\ \{\bx_s-\bx^h_s\}_{s=1}^S,\ \bW-\hat{\bW}^h,\ \bb-\hat{\bb}^h,\  \{\bz_s-\bz^h_s\}_{s=1}^S} +R.
\label{galerkin_step}
\end{aligned}
\end{equation}     
This replacement is valid because:
\[\bW-\bW^h=\LRp{\bW-\hat{\bW}^h}+\LRp{\hat{\bW}^h-\bW^h},\]
and the term involving $\LRp{\hat{\bW}^h-\bW^h}$ vanishes due to Galerkin orthogonality (since $\hat{\bW}^h-\bW^h\in \LRp{\sUw}_h$). The same argument applies to the bias terms.
Now evaluating different terms in \eqref{galerkin_step} based on the definition \eqref{lag_fem} we have:
\[\tilde{\mathcal{J}}\LRp{\{\bx_s(t_T)\}_{s=1}^S}- \tilde{\mathcal{J}}\LRp{\{\bx^{h}_s(t_T)\}_{s=1}^S}=\sum_{n=1}^{T-1}\LRp{\frac{1}{2} \rho_w^n+\frac{1}{2}\rho_b^n}+R, \]
where
\[ \rho_{w}^n=-\sum_{s=1}^S\int_{I_n}\LRp{\bW(t)-\hat{\bW}^h(t),\ \LRs{\frac{\partial{\sigma(\bW^{h}(t)\bx^{h}_s(t)+\bb^{h}(t))}}{\partial \bW^{h}(t)}}^T\bz^{h}_s(t)}\ dt \]
\[ \rho_{b}^n=-\sum_{s=1}^S\int_{I_n}\LRp{\bb(t)-\hat{\bb}^h(t),\ \LRs{\frac{\partial{\sigma(\bW^{h}(t)\bx^{h}_s(t)+\bb^{h}(t))}}{\partial \bb^{h}(t)}}^T\bz^{h}_s(t)}\ dt, \]
 and the other terms vanish due to \eqref{forward_prop_fem}, and \eqref{adjoint_fem}.
 This concludes the proof.

 \end{proof}

\begin{remark}
\label{intention}
 The choice of strong form formulation for the state equation \eqref{forward_prop_fem} and adjoint equation \eqref{adjoint_fem} eliminates two additional residual terms that would otherwise appear in the error estimate under a finite element discretization. This simplification is intentional: by avoiding finite element discretization at the theoretical level, we obtain a cleaner error characterization while retaining the flexibility to employ efficient explicit schemes (such as forward Euler) in the practical implementation. The discretization parameter K (see section \ref{discretize_state}) controls the accuracy of the numerical approximation of states and adjoints, and can be increased to reduce discretization error as needed.
\end{remark}
\begin{corollary}
\label{cor_one}
Consider the assumptions and the result \eqref{main_result} in Theorem \ref{theo_one}.  We have:
\begin{equation}
    \snor{\tilde{\mathcal{J}}\LRp{\{\bx_s(t_T)\}_{s=1}^S}- \tilde{\mathcal{J}}\LRp{\{\bx^{h}_s(t_T)\}_{s=1}^S}}\leq  \sum_{n=1}^{T-1}  \underbrace{\frac{1}{2}\LRp{\omega_n^wR_n^w+\omega_n^bR_n^b}}_{\sE_n}+\snor{R},
    \label{error_est}
\end{equation}
where
\begin{align}
\omega_n^w&=h_n \norm{\bW(t)-\hat{\bW}^h(t)}_{I_n},\quad \omega_n^b=h_n \norm{\bb(t)-\hat{\bb}^h(t)}_{I_n}, \label{omega_def}\\
R_n^w&=\sum_{s=1}^S\norm{\LRs{\frac{\partial{\sigma(\bW^{h}(t)\bx_s^{h}(t)+\bb^{h}(t))}}{\partial \bW^{h}(t)}}^T\bz_s^{h}(t)}_{I_n},R_n^b=\sum_{s=1}^S\norm{\LRs{\frac{\partial{\sigma(\bW^{h}(t)\bx^h_s(t)+\bb^{h}(t))}}{\partial \bb^{h}(t)}}^T\bz^h_s(t)}_{I_n} \label{R_def}
\end{align}
Here the notation $\nor{f(t)}_{I_n}$ denotes the supremum norm on the interval $I_n$:
\begin{equation}
\nor{f(t)}_{I_n}=\sup_{t\in I_n} \nor{f(t)}_2.
\label{sup_norm}
\end{equation}
\end{corollary}
\begin{proof}
We derive the upper bound by applying the triangle inequality and Cauchy-Schwarz inequality to the error representation in Theorem \ref{theo_one}.
From the result \eqref{main_result} in Theorem \ref{theo_one}, taking absolute values on both sides and applying the triangle inequality we have:
\begin{equation}
 \snor{\tilde{\mathcal{J}}\LRp{\{\bx_s(t_T)\}_{s=1}^S}- \tilde{\mathcal{J}}\LRp{\{\bx^{h}_s(t_T)\}_{s=1}^S}}\leq \frac{1}{2}\sum_{n=1}^{T-1}\LRp{\snor{\rho_w^n}+\snor{\rho_b^n}} +\snor{R}.
 \label{triangle_step}
\end{equation}
Consider the term $\rho_w^n$  in Theorem \ref{theo_one}:
\begin{equation}
\rho_{w}^n=-\sum_{s=1}^S\int_{I_n}\LRp{\bW(t)-\hat{\bW}^h(t),\ \LRs{\frac{\partial{\sigma(\bW^{h}(t)\bx^{h}_s(t)+\bb^{h}(t))}}{\partial \bW^{h}(t)}}^T\bz^{h}_s(t)}\ dt.
\label{rho_w_integral}
\end{equation}
Taking the absolute value and applying the Cauchy-Schwarz inequality for the inner product, we have:
\begin{equation}
\snor{\rho_w^n}\leq \sum_{s=1}^S\int_{I_n}\norm{\bW(t)-\hat{\bW}^h(t)}_2\norm{\LRs{\frac{\partial{\sigma(\bW^{h}(t)\bx^{h}_s(t)+\bb^{h}(t))}}{\partial \bW^{h}(t)}}^T\bz^{h}_s(t)}_2\ dt.
\label{cs_applied_w}
\end{equation}
Since the integrand is non-negative, we can bound it by replacing the functions with their supremum values over the interval $I_n$:
\begin{equation}
\begin{aligned}
\snor{\rho_w^n}&\leq \sum_{s=1}^S\int_{I_n}\sup_{t\in I_n}\norm{\bW(t)-\hat{\bW}^h(t)}_2 \cdot \sup_{t\in I_n}\norm{\LRs{\frac{\partial{\sigma(\bW^{h}(t)\bx_s^{h}(t)+\bb^{h}(t))}}{\partial \bW^{h}(t)}}^T\bz_s^{h}(t)}_2\ dt\\
&= h_n \norm{\bW(t)-\hat{\bW}^h(t)}_{I_n} \cdot \LRp{\sum_{s=1}^S\norm{\LRs{\frac{\partial{\sigma(\bW^{h}(t)\bx_s^{h}(t)+\bb^{h}(t))}}{\partial \bW^{h}(t)}}^T\bz_s^{h}(t)}_{I_n}}=\omega_n^wR_n^w
\end{aligned}
\label{supremum_bound_w}
\end{equation}
Following similar procedure for $\rho_b^n$ in Theorem \ref{theo_one}, we have:
\begin{equation}
\snor{\rho_b^n}\leq  h_n \norm{\bb(t)-\hat{\bb}^h(t)}_{I_n} \LRp{\sum_{s=1}^S\norm{\LRs{\frac{\partial{\sigma(\bW^{h}(t)\bx^h_s(t)+\bb^{h}(t))}}{\partial \bb^{h}(t)}}^T\bz^h_s(t)}_{I_n}}=\omega_n^bR_n^b.
\label{omega_R_product_b}
\end{equation}
Substituting the bounds \eqref{supremum_bound_w} and \eqref{omega_R_product_b} into \eqref{triangle_step}:
\begin{equation}
 \snor{\tilde{\mathcal{J}}\LRp{\{\bx_s(t_T)\}_{s=1}^S}- \tilde{\mathcal{J}}\LRp{\{\bx^{h}_s(t_T)\}_{s=1}^S}}\leq \frac{1}{2}\sum_{n=1}^{T-1}\LRp{\omega_n^wR_n^w+\omega_n^bR_n^b}+\snor{R},
  \label{final_bound}
\end{equation}
which is the desired result \eqref{error_est}, thereby concluding the proof.
\end{proof}

\subsection{Approximating $\omega_n^w$ and $\omega_n^b$ using interpolation estimates}
\label{sec:omega_approx}

To compute the error bound in \eqref{error_est}, we require estimates of the terms $\omega_n^w$ and $\omega_n^b$, which depend on the true (but unavailable) solutions $\bW(t)$ and $\bb(t)$. Since the true solution is unknown, we employ standard interpolation error estimates to approximate these quantities.
By classical finite element interpolation theory, there exist $\bar{\bW}^h(t)\in \LRp{\sUw}_h$ and $\bar{\bb}^h(t)\in \LRp{\sUb}_h$ such that:
\begin{equation}
  \norm{\bW(t)-\bar{\bW}^h(t)}_{I_n}\leq \frac{h_n^2}{8}\norm{\ddot{\bW}(t)}_{I_n},\quad \norm{\bb(t)-\bar{\bb}^h(t)}_{I_n}\leq \frac{h_n^2}{8}\norm{\ddot{\bb}(t)}_{I_n},
  \label{interp_est}
\end{equation}
where the second derivatives $\ddot{\bW}(t)$ and $\ddot{\bb}(t)$ are approximated in practice using centered difference quotients of the discrete weight and bias values \cite{kraft2010dual}. These approximations provide practical estimates for $\omega_n^w$ and $\omega_n^b$ that can be computed from the available discrete solution.

\subsection{Computing $R_n^w$ and $R_n^b$ in Corollary \ref{cor_one}}
\label{discretize_state}
Note that the terms $R_n^w$ and $R_n^b$ in Corollary \ref{cor_one} involve the state variables $\bx^h(t)$ and adjoint variables $\bz^h(t)$, which theoretically satisfy the ODEs \eqref{forward_prop_fem} and \eqref{adjoint_fem}.
While the theoretical error estimate in Corollary \ref{cor_one} assumes that the state equation \eqref{forward_prop_fem} and adjoint equation \eqref{adjoint_fem} are satisfied exactly, in practice these ODEs must be discretized numerically. We employ a forward Euler scheme with a fine sub-discretization (controlled by a parameter $K$) to control the discretization error (see Figure \ref{fig:discretization}). We assume that $K$ is chosen such that the state/adjoint discretization error is very small and the error estimate in Corollary \ref{cor_one} solely focuses on the parameter representation error (see remark \ref{motiv_remark}).

\subsubsection{Discretization strategy}
\label{discrete_sect}
Consider the mesh $t_1<t_2<\cdots<t_T$ defining the network layers, with step sizes $h_{k} =t_{k+1}-t_{k}$ and intervals $I_{k}=(t_{k},t_{k+1})$ for $k=1,\dots, T-1$. To accurately solve the ODEs, we introduce a sub-discretization parameter $K\in \mathbb{N}$, which subdivides each interval $I_k$ into $K$ smaller steps of size $h_k/K$. For each interval $I_k = (t_k, t_{k+1})$, we subdivide it into $K$ equal subintervals with step size $\Delta t = h_k/K$. Define the substep time points:
\begin{equation}
t_k^{(r)} = t_k + \frac{r h_k}{K}, \quad r=0,1,\ldots,K.
\label{substep_times}
\end{equation}
\subsubsection{Proposed neural network architecture}
\label{our_archi}
Let us denote  $\tilde{\bW}_0,\ \tilde{\bb}_0,\ \tilde{\bW}^h(t),\ \tilde{\bb}^h(t), \ \tilde{\bW}_T,\ \tilde{\bb}_T$ as the weights and biases of our proposed architecture, $\tilde{\bx}_s^h(t)$ denotes the state and $\tilde{\bz}_s^h(t)$ denotes the corresponding adjoint variable. Then, the forward propagation of our architecture can be written as follows:
\paragraph{Forward propagation for our proposed neural network architecture:}
\begin{equation}
\mathrm{\bold{Input \ layer:}} \ \quad \tilde{\bx}_s^h(t_1)=\sigma_1(\tilde{\bW}_0\bx_{s;0}+\tilde{\bb}_0),
\label{input_l}
\end{equation}
 The forward Euler scheme for the {\bf{hidden layer}} ODE \eqref{forward_prop_fem}  with a fine discretization for the states  is  written as:
\begin{equation}
\tilde{\bx}^h_s\LRp{t_k^{(r)}}=\tilde{\bx}^h_s\LRp{t_k^{(r-1)}} +\LRp{\frac{h_k}{K}}\sigma\LRp{\tilde{\bW}^h\LRp{t_k^{(r-1)}}\tilde{\bx}^h_s\LRp{t_k^{(r-1)}}+\tilde{\bb}^h\LRp{t_k^{(r-1)}} },
\label{discrete}
\end{equation}
for $r=1,\dots, K$ and $k=1,\dots, T-1$. 
The weights and biases $\tilde{\bW}^h(t)$, $\tilde{\bb}^h(t)$ are evaluated at each substep via linear interpolation:
\begin{equation}
\begin{aligned}
 &\tilde{\bW}^h(t)=\LRp{1-\frac{t-t_{k}}{h_{k}}}\tilde{\bW}^h(t_{k})+\LRp{\frac{t-t_{k}}{h_{k}}}\tilde{\bW}^h(t_{k+1}),\quad t_{k}\leq t\leq t_{k+1},\\
 & \tilde{\bb}^h(t)=\LRp{1-\frac{t-t_{k}}{h_{k}}}\tilde{\bb}^h(t_{k})+\LRp{\frac{t-t_{k}}{h_{k}}}\tilde{\bb}^h(t_{k+1}), \quad t_{k}\leq t\leq t_{k+1}.
    \end{aligned}
    \label{rep_new}
\end{equation}
\begin{equation}
    \mathrm{\bold{Output \ layer:}} \ \quad \tilde{\bx}_s^h(t_{T+1})=\sigma_{T+1}(\tilde{\bW}_T\tilde{\bx}_{s}(t_T)+\tilde{\bb}_T),
    \label{output_l}
\end{equation}
Figure \ref{fig:discretization} illustrates the two-level discretization scheme: the coarse mesh $\{t_1,\ldots,t_T\}$ defines the network layers where parameters are stored, while the fine sub-mesh (controlled by parameter $K$) is used to accurately propagate states and adjoints between layers.
\paragraph{Adjoint propagation discretization.}
Similarly, the adjoint equation \eqref{adjoint_fem} is discretized backward in time using an analogous backward Euler scheme to  estimate $\tilde{\bz}^h_s\LRp{t_k^{(r)}}$.

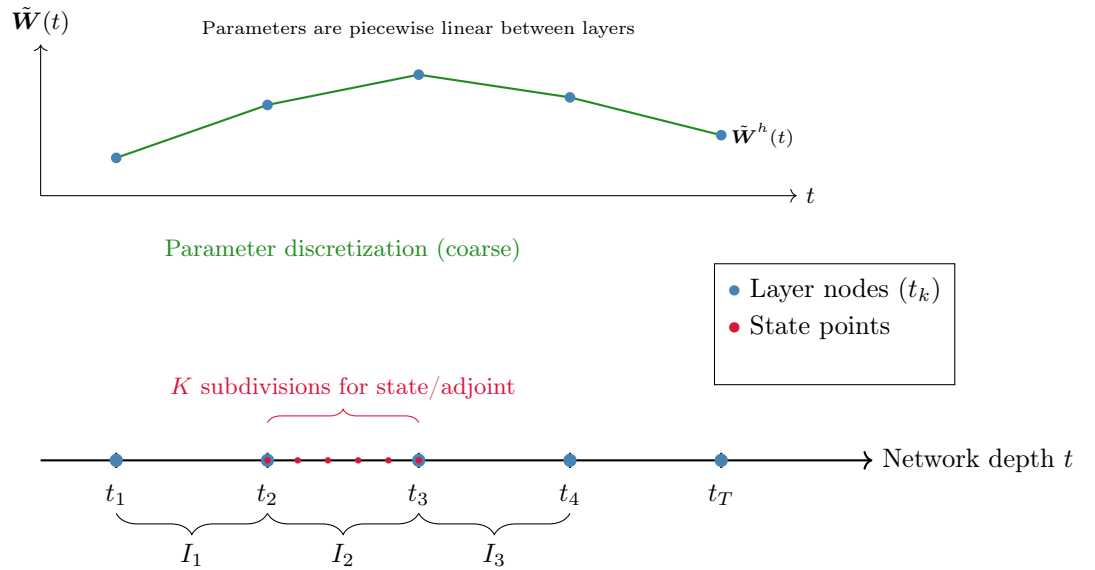
\begin{figure}[h!]
\centering
\begin{tikzpicture}[scale=1.0]
    \definecolor{layercolor}{RGB}{70,130,180}
    \definecolor{statecolor}{RGB}{220,20,60}
    \definecolor{paramcolor}{RGB}{34,139,34}
    
    \draw[thick,->] (0,0) -- (11,0) node[right] {Network depth $t$};
    
    \foreach \x/\label in {1/t_1, 3/t_2, 5/t_3, 7/t_4, 9/t_T} {
        \draw[thick] (\x,0.1) -- (\x,-0.1);
        \node[below] at (\x,-0.2) {$\label$};
        \fill[layercolor] (\x,0) circle (2.5pt);
    }
    
    \draw[decorate,decoration={brace,amplitude=8pt,mirror}] 
        (1,-0.7) -- (3,-0.7) node[midway,below=8pt] {$I_1$};
    \draw[decorate,decoration={brace,amplitude=8pt,mirror}] 
        (3,-0.7) -- (5,-0.7) node[midway,below=8pt] {$I_2$};
    \draw[decorate,decoration={brace,amplitude=8pt,mirror}] 
        (5,-0.7) -- (7,-0.7) node[midway,below=8pt] {$I_3$};
    
    \node[above,paramcolor,font=\small] at (4,2.5) {Parameter discretization (coarse)};

    \foreach \x in {3.0, 3.4, 3.8, 4.2, 4.6, 5.0} {
        \fill[statecolor] (\x,0) circle (1.2pt);
    }
    
    \draw[decorate,decoration={brace,amplitude=5pt},statecolor] 
        (3,0.5) -- (5,0.5) node[midway,above=5pt,font=\small] {$K$ subdivisions for state/adjoint};
    

    \node[draw,fill=white,inner sep=5pt,align=left] at (10.5,1.8) {
        \textcolor{layercolor}{$\bullet$} Layer nodes $(t_k)$ \\[2pt]
        \textcolor{statecolor}{$\bullet$} State points \\[2pt]
    };
    
    \begin{scope}[yshift=3.5cm]
        \draw[->] (0,0) -- (10,0) node[right,font=\small] {$t$};
        \draw[->] (0,0) -- (0,2) node[above,font=\small] {$\tilde{\bW}(t)$};
        
        \draw[thick,paramcolor] (1,0.5) -- (3,1.2) -- (5,1.6) -- (7,1.3) -- (9,0.8)
            node[right,font=\scriptsize,black] {$\tilde{\bW}^h(t)$};
        
        \foreach \x/\y in {1/0.5, 3/1.2, 5/1.6, 7/1.3, 9/0.8} {
            \fill[layercolor] (\x,\y) circle (2pt);
        }
        
        \node[font=\scriptsize] at (5,2.2) {Parameters are piecewise linear between layers};
    \end{scope}
\end{tikzpicture}
\caption{Two-level discretization scheme. \textbf{Coarse mesh (bottom):}  Nodes $\{t_1,\ldots,t_T\}$ (blue circles) define intervals $I_k$. Parameters $\tilde{\bW}^h(t_k), \tilde{\bb}^h(t_k)$ are stored at these nodes. \textbf{Fine mesh (bottom):} Within each interval, $K$ forward Euler substeps (red points) are used to propagate states and adjoints with step size $\Delta t = h_k/K$. \textbf{Linear interpolation (top):} The piecewise linear function $\tilde{\bW}^h(t)$ (green) shows how parameters are interpolated between  nodes for evaluation at substep points.}
\label{fig:discretization}
\end{figure}

\subsubsection{Computational error estimate}
\label{procedure_n}

In our work, we train a neural network of the form \eqref{discrete} to optimize for the weights $\tilde{\bW}^h(t_1),\dots \tilde{\bW}^h(t_T)$ and biases $\tilde{\bb}^h(t_1),\dots \tilde{\bb}^h(t_T)$.
Let $\tilde{\bx}^h_s(t)$ and $\tilde{\bz}^h_s(t)$ denote the piecewise linear interpolants of the discrete solutions obtained from \eqref{discrete}. Inorder to obtain  a practical estimate of the error bound in \eqref{error_est}, we substitute the computed approximations in \eqref{error_est}:
\[ \bx^h_s(t)\approx\tilde{\bx}^h_s(t),\quad \bz^h_s(t)\approx\tilde{\bz}^h_s(t),\quad \bW^h(t)=\tilde{\bW}^h(t),\quad \bb^h(t)=\tilde{\bb}^h(t). \]
The upper bound in \eqref{interp_est} is used to compute $\omega_n^w$ and $\omega_n^b$ (see section \ref{sec:omega_approx}).
\begin{algorithm}[h!] 
	\caption{A-posteriori error estimation and architecture  adaptation}
	\hspace*{\algorithmicindent} \textbf{Input}: Training data $\bold{X}$, labels $\bold{C}$, validation data $\bold{X}_1$, validation labels $\bold{C}_1$, number of neurons in each hidden layer $n_1$, loss function $\tilde{\mathcal{J}}$, number of iterations $N_n$, hyperparameter $K$, hyperparameter $T$, hyperparameters $t_1,\ t_T$,  hyperparameters for optimizer (see Appendix \ref{hyper_parameter_n}).\\
	\hspace*{\algorithmicindent} \textbf{Initialize}:   Initialize  network $\mathcal{Q}_1$ with $T$ hidden layers and nodes at $\{t_1,\ldots,t_T\}$, with hidden layer propagation given by \eqref{discrete}.\\
	\begin{algorithmic}[1] 
  \State Train network $\mathcal{Q}_1$ and freeze the parameters ($\tilde{\bW}_0^h,\ \tilde{\bb}_0^h,\ \tilde{\bW}_T^h,\ \tilde{\bb}_T^h$) of input layer \eqref{input_l} and output layer \eqref{output_l}. 
  \State Store the best validation loss $(\epsilon_v)^{1}$ and the best performing network $\mathcal{Q}_1$.
  		\State set $i =1$,  $(\epsilon_v)^{0}>>(\epsilon_v)^{1}$
		\While{$i \le N_n$ \textbf{and} $\LRs{(\epsilon_v)^{i}\leq (\epsilon_v)^{i-1}}$ } 
 \State Compute the a-posteriori error $\sE_n$ in \eqref{error_est} for each $n=1,\dots T-1$  based on the procedure in section \eqref{procedure_n}. 
 \State Compute $n^*=\arg \max_n\{ \sE_n\}$,
        \State Obtain the new network $\mathcal{Q}_{i+1}$ by adding a new layer/node at time  $t=\frac{t_{n^*}+t_{n^*+1}}{2}$ with weight initialized as  $\frac{\tilde{\bW}^h(t_{n^*})+\tilde{\bW}^h(t_{n^*+1})}{2}$, and bias initialized as  $\frac{\tilde{\bb}^h(t_{n^*})+\tilde{\bb}^h(t_{n^*+1})}{2}$.
        \State Set $T=T+1$ and form new nodes $\{t_1,\dots t_T\}$ after insering the new node.
        \State Train network $\mathcal{Q}_{i+1}$  and store the best validation loss $(\epsilon_v)^{i+1}$ and the best network $\mathcal{Q}_{i+1}$.
		\State $i = i+1$
		\EndWhile
        \State Defreeze parameters ($\tilde{\bW}_0^h,\ \tilde{\bb}_0^h,\ \tilde{\bW}_T^h,\ \tilde{\bb}_T^h$) of input layer and output layer and train the network $\mathcal{Q}_{i-1}$ to achieve further improvement on validation loss.
	\end{algorithmic} \label{Algo_full}
\hspace*{\algorithmicindent} \textbf{Output}: Network $\mathcal{Q}_{i-1}$
\end{algorithm}

Algorithm \ref{Algo_full} shows our proposed architecture adaptation algorithm. The procedure starts with training a network $\mathcal{Q}_1$ with $T$ hidden layers and nodes at $\{t_1,\ldots,t_T\}$. The network $\mathcal{Q}_1$ has a forward propagation given by \eqref{input_l}, \eqref{discrete}, and \eqref{output_l}. Once the network $\mathcal{Q}_1$ is trained, we freeze (make it untrainable/fixed) the parameters of the input and output layer and focus our attention on adapting the hidden layers to further improve the performance. To that end, we compute the error $\{\sE_n\}_{n=1}^{T-1}$ in \eqref{error_est}  based on the procedure in section \eqref{procedure_n}. A new layer is inserted at the location of the maximum error.
The weights and biases of the new layer are initialized as discussed in line 7 of Algorithm \ref{Algo_full}. The new network is trained again and the procedure is repeated until no more improvement in validation loss is achieved or the maximum depth is attained (line 4). Once the algorithm terminates, as a postprocessing step, we defreeze (make it trainable) the parameters of the input and output layer and retrain the network to achieve further improvement on validation loss. 
\begin{remark}
Note that a key assumption in Theorem \ref{theo_one} is that the weights and biases of the network must be at a local minimum (satisfying the first-order optimality condition). However, in Algorithm \ref{Algo_full} we use a validation dataset to retain the best performing network (see lines 2 and 9 of Algorithm \ref{Algo_full}), which means that the network is not necessarily at a local minimum with respect to the training dataset at any given iteration of the algorithm. If the parameters are very far from a local minimum, this would impact the accuracy of the error estimated using Corollary \ref{cor_one}.
\end{remark}
\section{Numerical demonstration}
\label{num_d}
In this section, we numerically demonstrate the proposed approach on several datasets, with a primary focus on regression tasks.  General experimental settings for all the problems and descriptions of methods adopted for
comparison are detailed in Appendix \ref{hyper_parameter}. Note that in our numerical results, ``proposed approach" refers to Algorithm \ref{Algo_full}. Details of hyperparameter values used for different problems is provided in Appendix \ref{hyper_parameter_n}.  The loss function $\mathcal{J}$ in \eqref{stat_final_true} is defined as the mean-squared error between the network output and the corresponding target data (label).

\subsection{Proof of concept example}
\label{proof_of}
As a proof of concept of our proposed approach, we first consider the problem of learning a 2-dimensional function. The goal is to learn a nonlinear function $f(x,\ y)=e^{-0.1\LRp{x^2+y^2}}\times \sin{x}\times \cos{y}$ on the domain $[-5,\ 5]\times [-5,\ 5]$ as accurately as possible. The contours of this function is shown in subfigure (f) of Figure \ref{evolution_toy}. We generated $1000$ data points from a  uniform distribution over the domain for training and considered $500$ data points for the testing and validation respectively. In this case we have the input dimension $n_0=2$ and output dimension for the network $n_{T+1}=1$.
Our adaptation  procedure starts by training a three hidden layer network with just $5$ neurons in each hidden layer and we progressively increase the depth based on Algorithm \ref{Algo_full}.
\subsubsection{Estimating discretization parameter $K$}
One of the key hyperparameters used in Algorithm \ref{Algo_full} is the parameter $K$. Note that the parameter $K$ is a discretization parameter  chosen to ensure that the state and adjoint discretization error are small (see section \ref{discrete_sect}).  We conduct a mesh convergence study to determine an appropriate value for $K$. To that end, let us first define an error indicator as follows:
\[ I_{\epsilon}=\frac{1}{S\times T}\sum_{i=1}^{T}\sum_{s=1}^S \frac{\norm{\tilde{\bx}_s^h(t_i)- \tilde{\bx}_s(t_i)}^2}{\norm{\tilde{\bx}_s(t_i)}^2},\]
where $\tilde{\bx}_s^h(t_i)$  is the solution  obtained after training network in section \ref{our_archi} for a given discretization parameter $K$. 
 $\tilde{\bx}_s(t_i)$ denotes an estimate of the true solution obtained after training network in section \ref{our_archi} for a choice of large $K$ (in this case $K=1000$).

We conduct the mesh convergence study for the initial network  with $T=3$ (three hidden layers). Figure \ref{fig:mesh_study} (left subfigure) shows how the average relative error $I_{\epsilon}$ decreases as $K$ increases. We choose $K=4$ such that the relative error $I_{\epsilon}$ is $0.03$. This choice is sufficient to guarantee low error in the state/adjoint discretization for the initial coarse network, leading to accurate computational estimates (upper bounds) for the a-posteriori error in \eqref{error_est} (see Remark \ref{discretize_state} for an explanation of why low error in state/adjoint discretization leads to accurate computational estimates for the a-posteriori error).

Figure \ref{fig:mesh_study} (left subfigure) also shows that the computational time for training increases with $K$. Since larger $K$
 leads to slower forward propagation for our network (section \ref{our_archi}), we refrain from using large $K$  values for computational efficiency.

\subsubsection{A-posteriori error estimates in Corollary \ref{cor_one}}

To see how well the estimated upper bound  in  \eqref{error_est} correlates with the true error in  \eqref{error_est}, we plot the true error 
 $\snor{\tilde{\mathcal{J}}\LRp{\{\bx_s(t_T)\}_{s=1}^S}- \tilde{\mathcal{J}}\LRp{\{\bx^{h}_s(t_T)\}_{s=1}^S}}$ vs. the estimated upper bound in \eqref{error_est} and the results are shown in Figure \ref{fig:mesh_study} (right subfigure). For estimating the true error $\snor{\tilde{\mathcal{J}}\LRp{\{\bx_s(t_T)\}_{s=1}^S}- \tilde{\mathcal{J}}\LRp{\{\bx^{h}_s(t_T)\}_{s=1}^S}}$, we assume that $\tilde{\mathcal{J}}\LRp{\{\bx_s(t_T)\}_{s=1}^S}\approx 0$, i.e there exists a true network in \eqref{stat_final_true} that both fits the training data well and generalizes well, achieving near-zero validation loss. Figure \ref{fig:mesh_study} (right panel) shows that the estimated upper bound $\sum_{n=1}^{T-1}\sE_n$ in  \eqref{error_est} decreases with each adaptation iteration and closely approximates the true error.
\begin{figure}[h!]
    \centering

    \begin{subfigure}[b]{0.45\textwidth}
        \centering
        \includegraphics[scale=0.5]{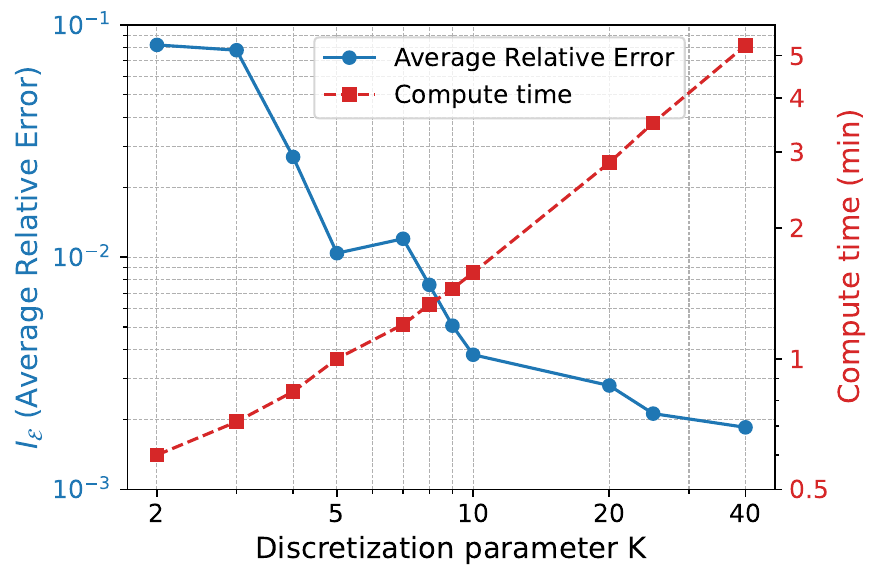}
        \caption{Estimating discretization parameter $K$}
    \end{subfigure}
    \hfill
    \begin{subfigure}[b]{0.45\textwidth}
        \hspace{-1 cm}
        \includegraphics[scale=0.5]{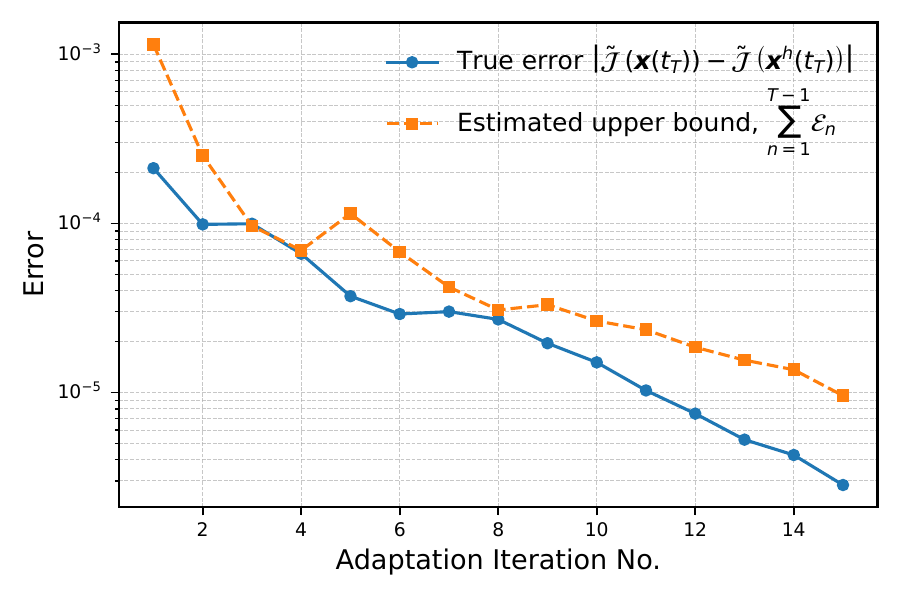}
        \caption{A-posteriori error estimates}
    \end{subfigure}
    \caption{Left to Right: Mesh convergence study for estimating discretization parameter $K$; A-posteriori error estimates at different iterations of the algorithm}
    \label{fig:mesh_study}
\end{figure}
An  important aspect of the derived upper bound in  \eqref{error_est}  is that it tells us how the total error $\sum_{n=1}^{T-1}\sE_n$  distributes across each interval $I_n$ ($\sE_n$ is the error in interval $I_n$). We can then use it for adapting the network  by adding a new layer (with new weights and biases) in the interval with maximum error as described in line 7 of Algorithm \ref{Algo_full}. Figure \ref{decomp} shows the decomposition of error $\sE_n$ computed using \eqref{error_est}  at different iterations of the algorithm. We clearly see that upon adding a new hidden layer at the location of maximum error (indicated as a red block in Figure \ref{decomp}) and retraining the network, the error decreases at the location. 

\begin{figure}[h!]
    \hspace{-0.6 cm}
        \includegraphics[scale=0.35]{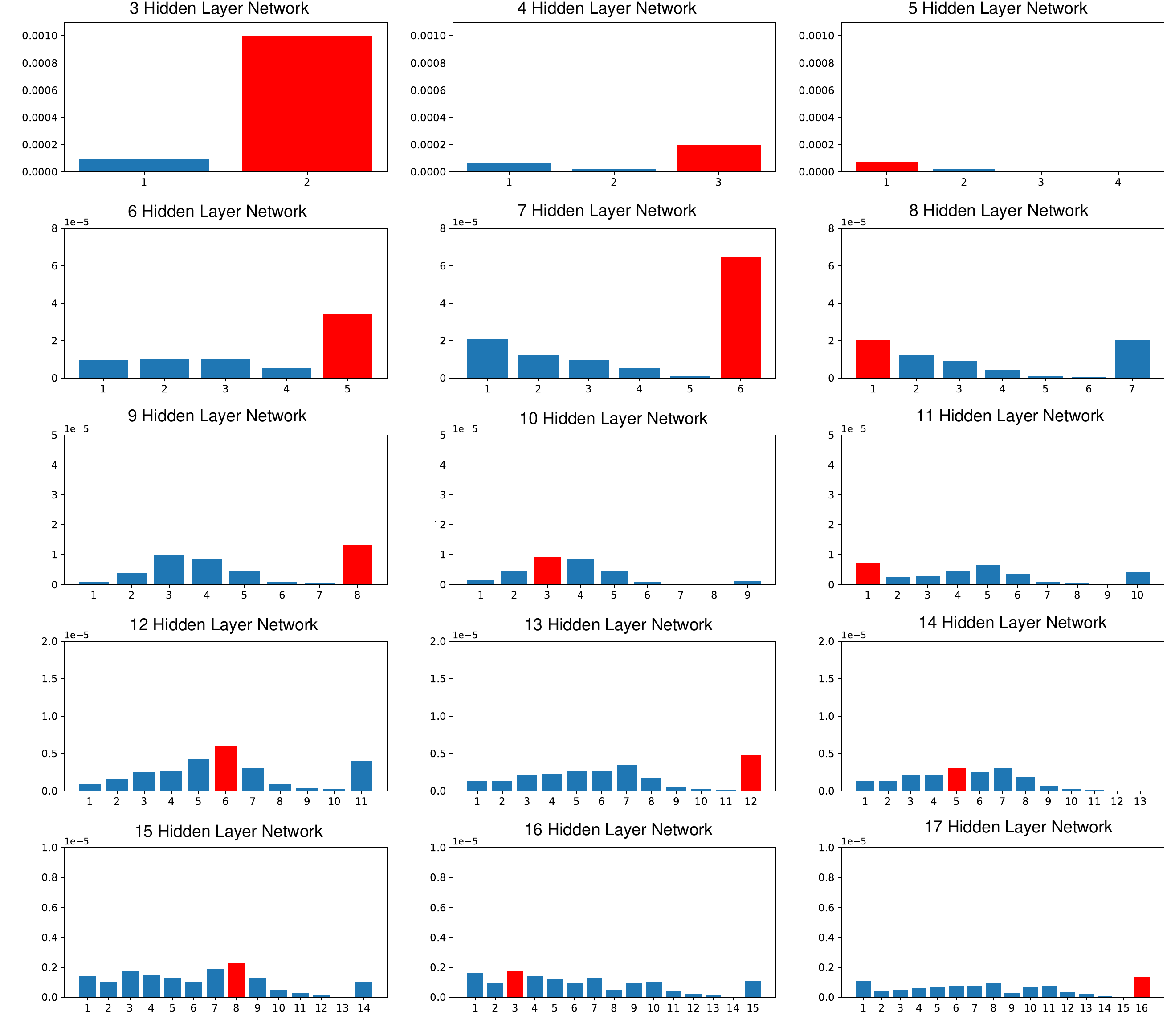}
    \caption{Decomposition of the a-posteriori error $\sE_n$ in \eqref{error_est} at different iterations of the algorithm (x-axis denotes interval number and y-axis denotes the magnitude of error). Block in red denotes the interval of maximum error where a new layer is added for the next training phase.}
    \label{decomp}
\end{figure}
\begin{figure}[h!]
\hspace{-0.2 cm}
    \begin{subfigure}[b]{0.3\textwidth}
        \includegraphics[scale=0.37]{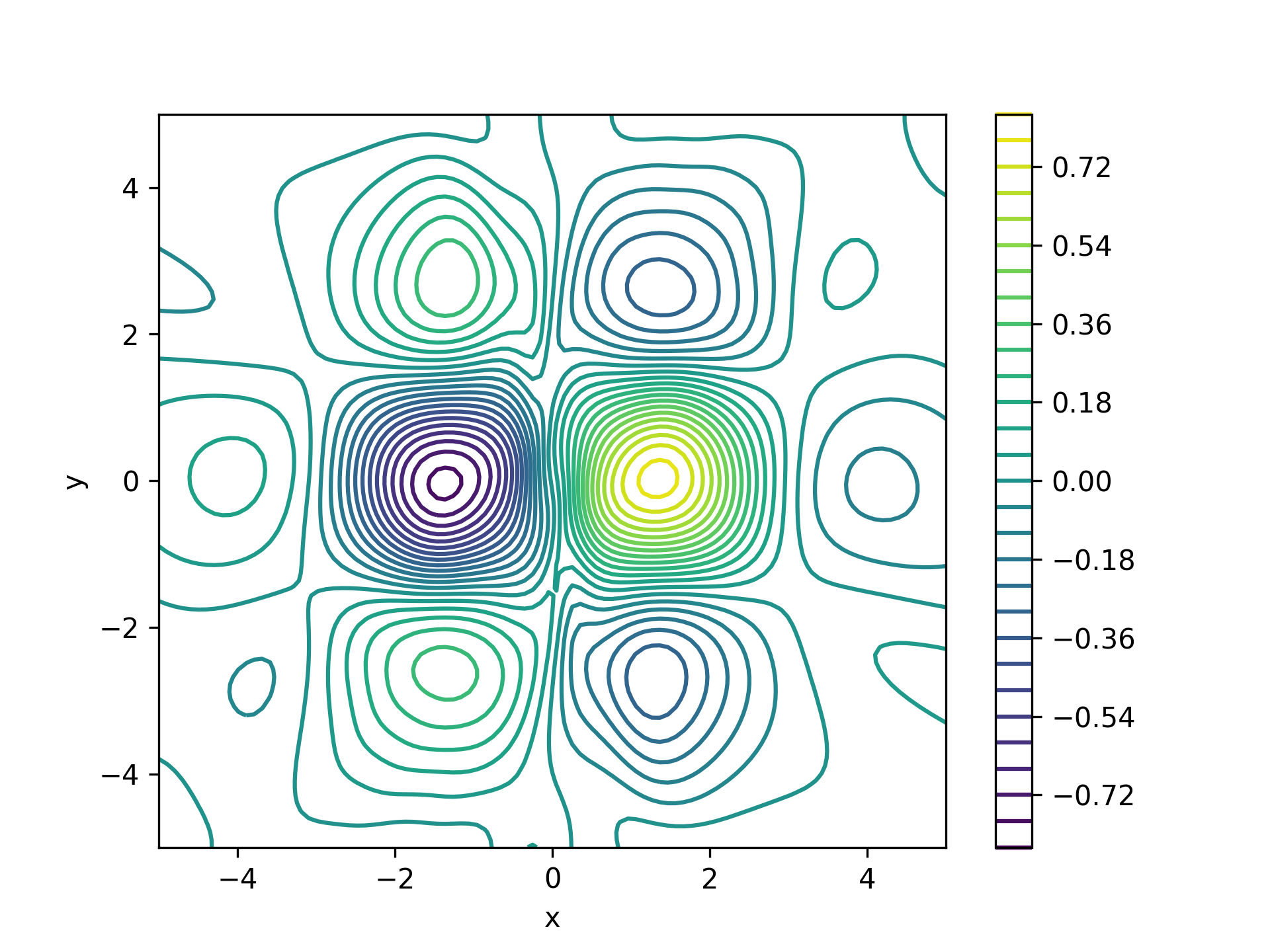}
        \caption{3 hidden layers}
    \end{subfigure}
 \hspace{-0.2 cm}
    \begin{subfigure}[b]{0.3\textwidth}
        \includegraphics[scale=0.37]{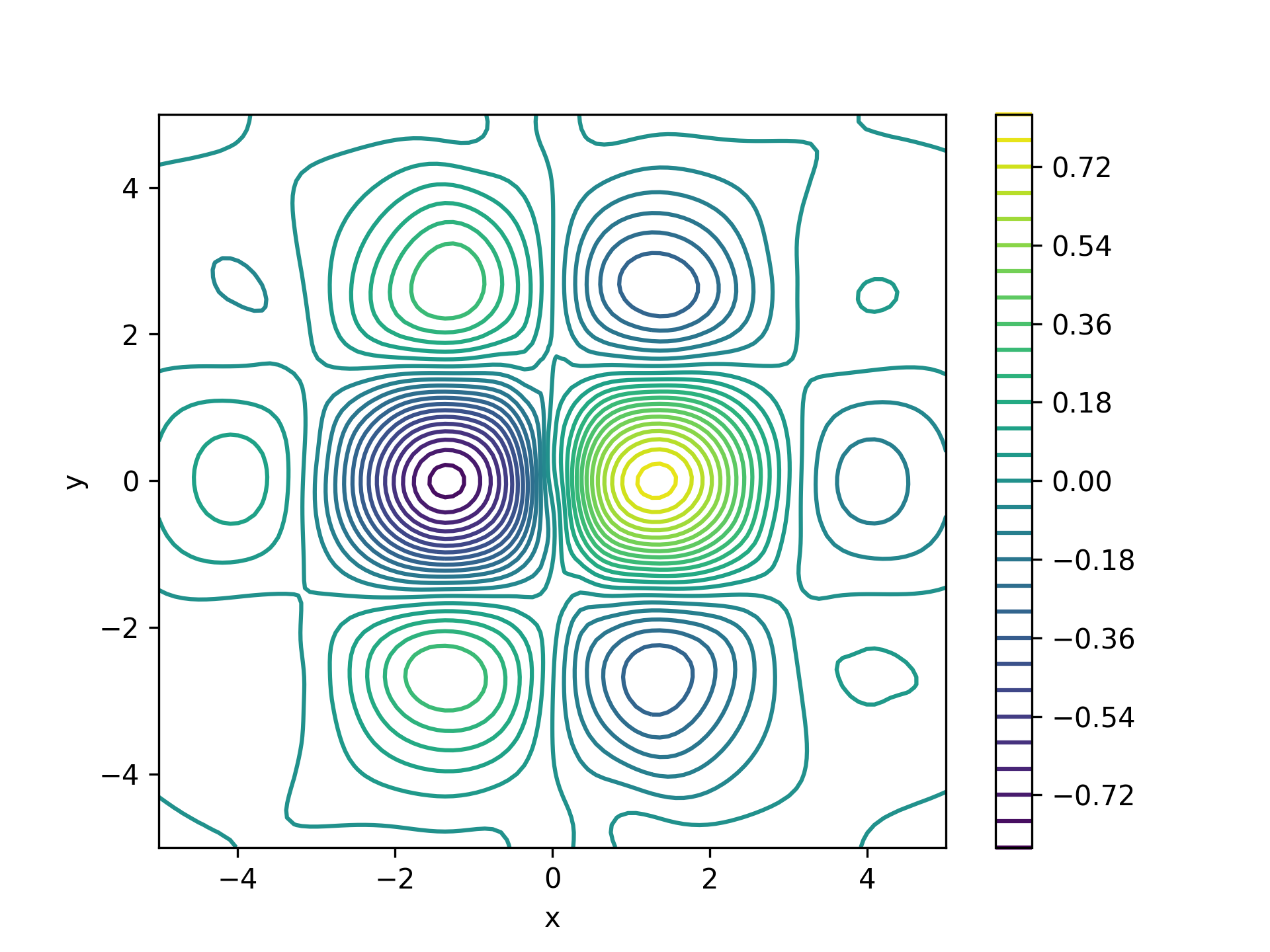}
        \caption{7 hidden layers}
    \end{subfigure}
    \hspace{-0.2 cm}
    \begin{subfigure}[b]{0.3\textwidth}
        \includegraphics[scale=0.37]{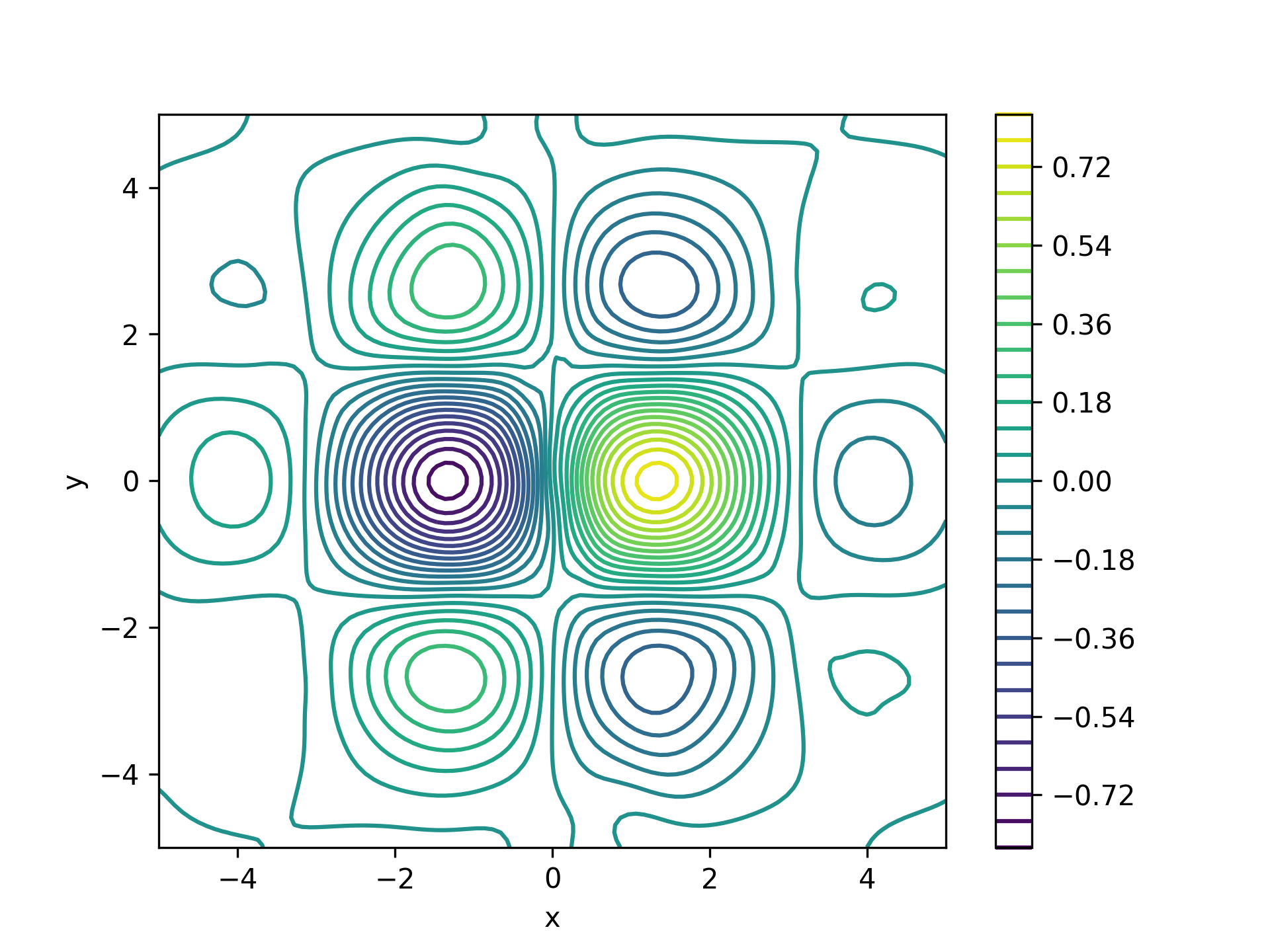}
        \caption{10 hidden layers}
    \end{subfigure}

    \vspace{0.5cm}
\hspace{-0.2 cm}
    \begin{subfigure}[b]{0.3\textwidth}
        \includegraphics[scale=0.37]{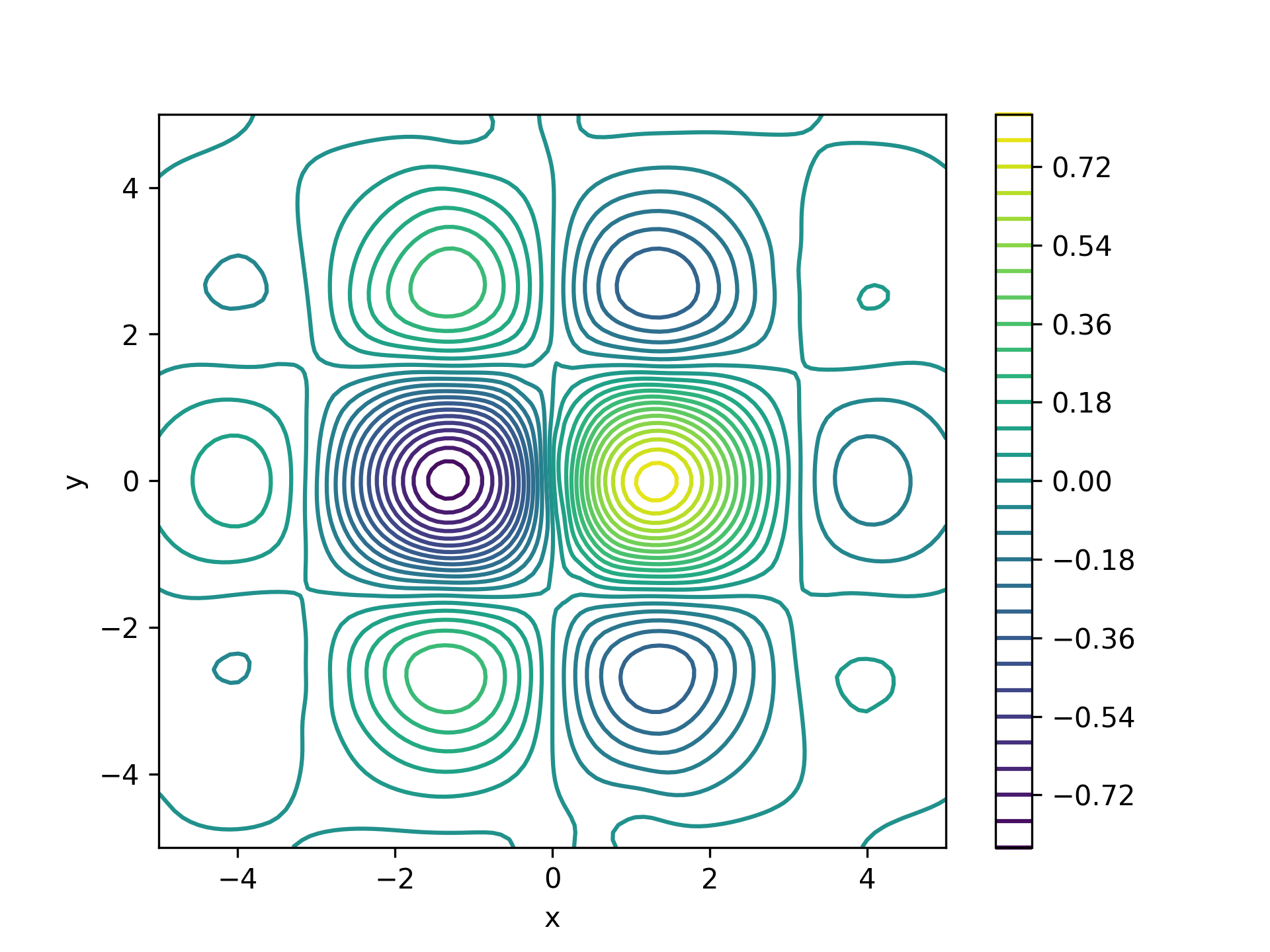}
        \caption{12 hidden layers}
    \end{subfigure}
 \hspace{-0.2 cm}
    \begin{subfigure}[b]{0.3\textwidth}
        \includegraphics[scale=0.37]{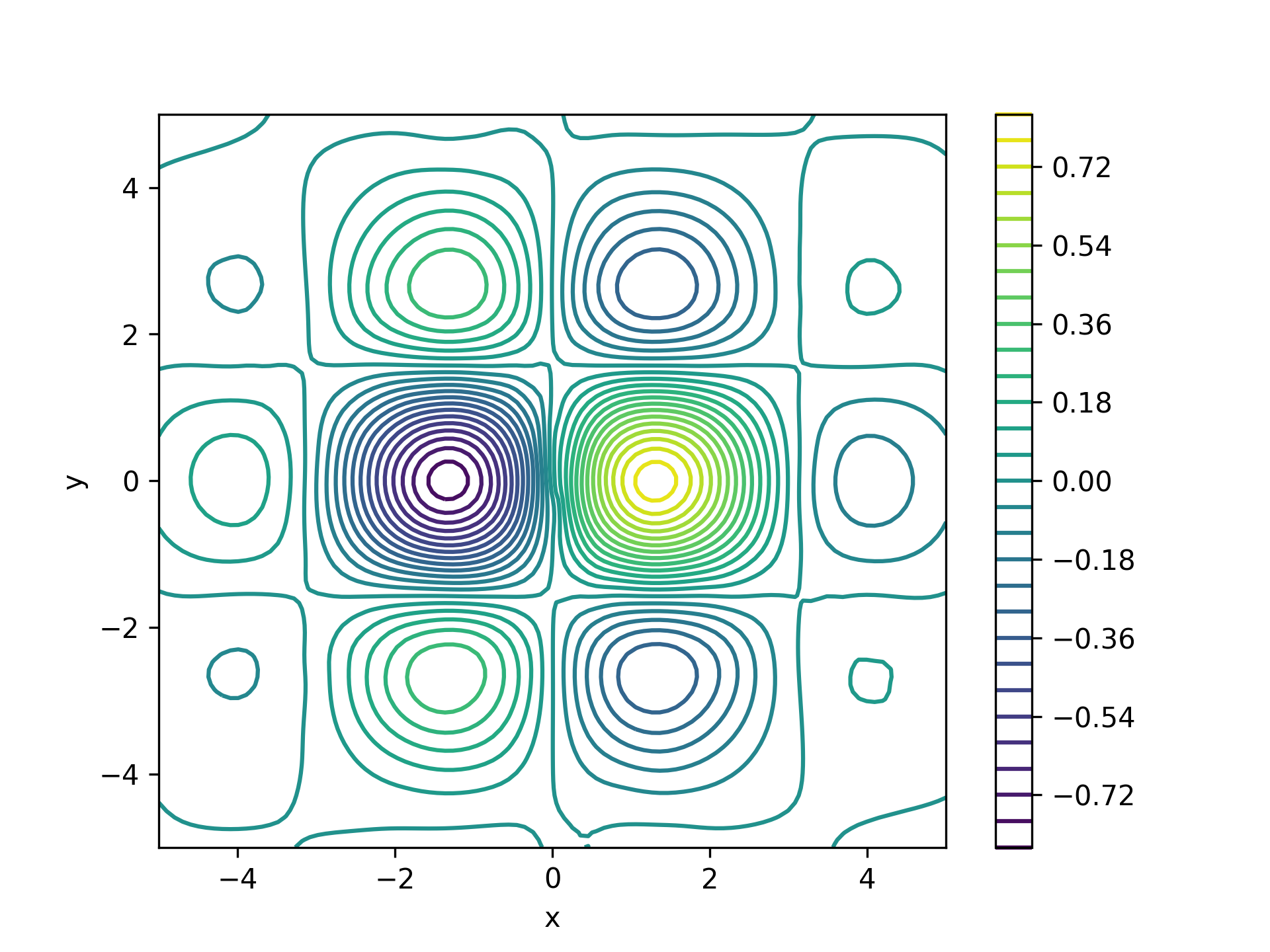}
        \caption{17 hidden layers}
    \end{subfigure}
    \hspace{-0.2 cm}
    \begin{subfigure}[b]{0.3\textwidth}
        \includegraphics[scale=0.37]{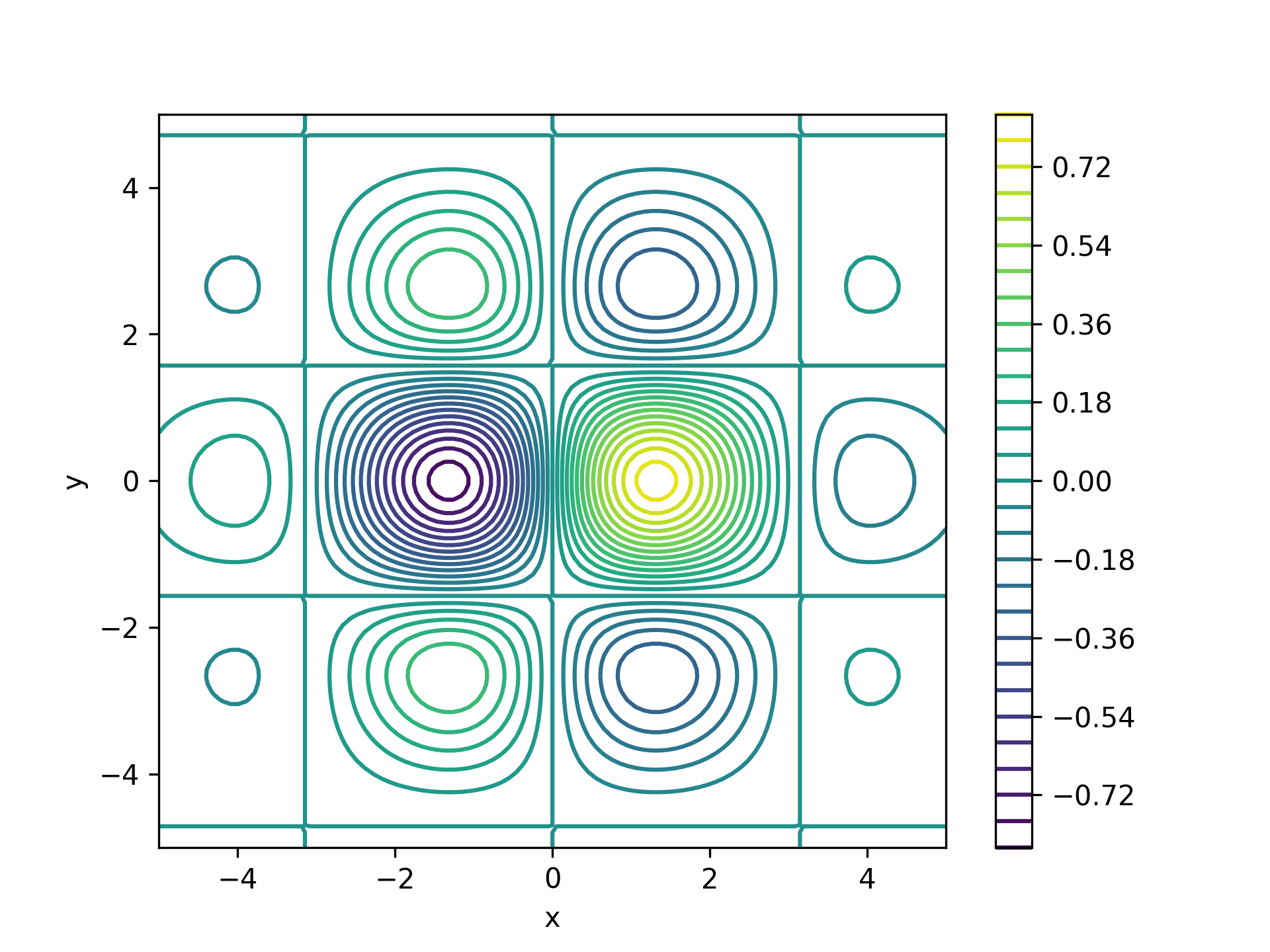}
        \caption{True solution}
    \end{subfigure}
\caption{Evolution of contours of the solution upon adding new hidden layers in our proposed framework. Network with $17$ hidden layers is able to learn the true function accurately.}
    \label{evolution_toy}
\end{figure}
\begin{figure}[h!]
\centering
\includegraphics[scale=0.4]{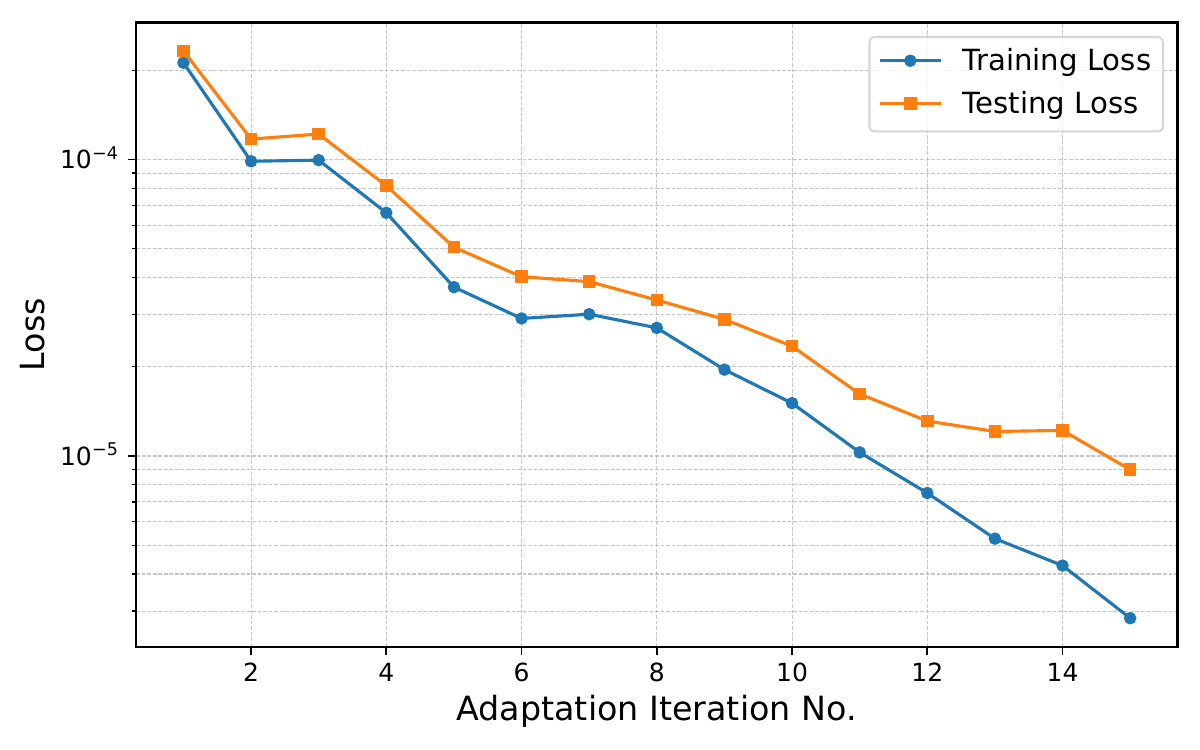}
\caption{Training loss and testing loss achieved at different iterations of the algorithm.}
\label{test_train}
\end{figure}

The improvement in solution on adding new layers is shown in Figure \ref{evolution_toy} where one clearly sees that the algorithm progressively picks up complex features in the solution and generalizes better as one adds more layers. Figure \ref{test_train} shows how the testing loss decreases with adaptation iteration. 
\begin{figure}[h!]
    \begin{subfigure}[b]{0.3\textwidth}
        \includegraphics[scale=0.37]{Figures/toy/proposed14.png}
          \caption{Proposed approach}
    \end{subfigure}
    \begin{subfigure}[b]{0.3\textwidth}
        \includegraphics[scale=0.37]{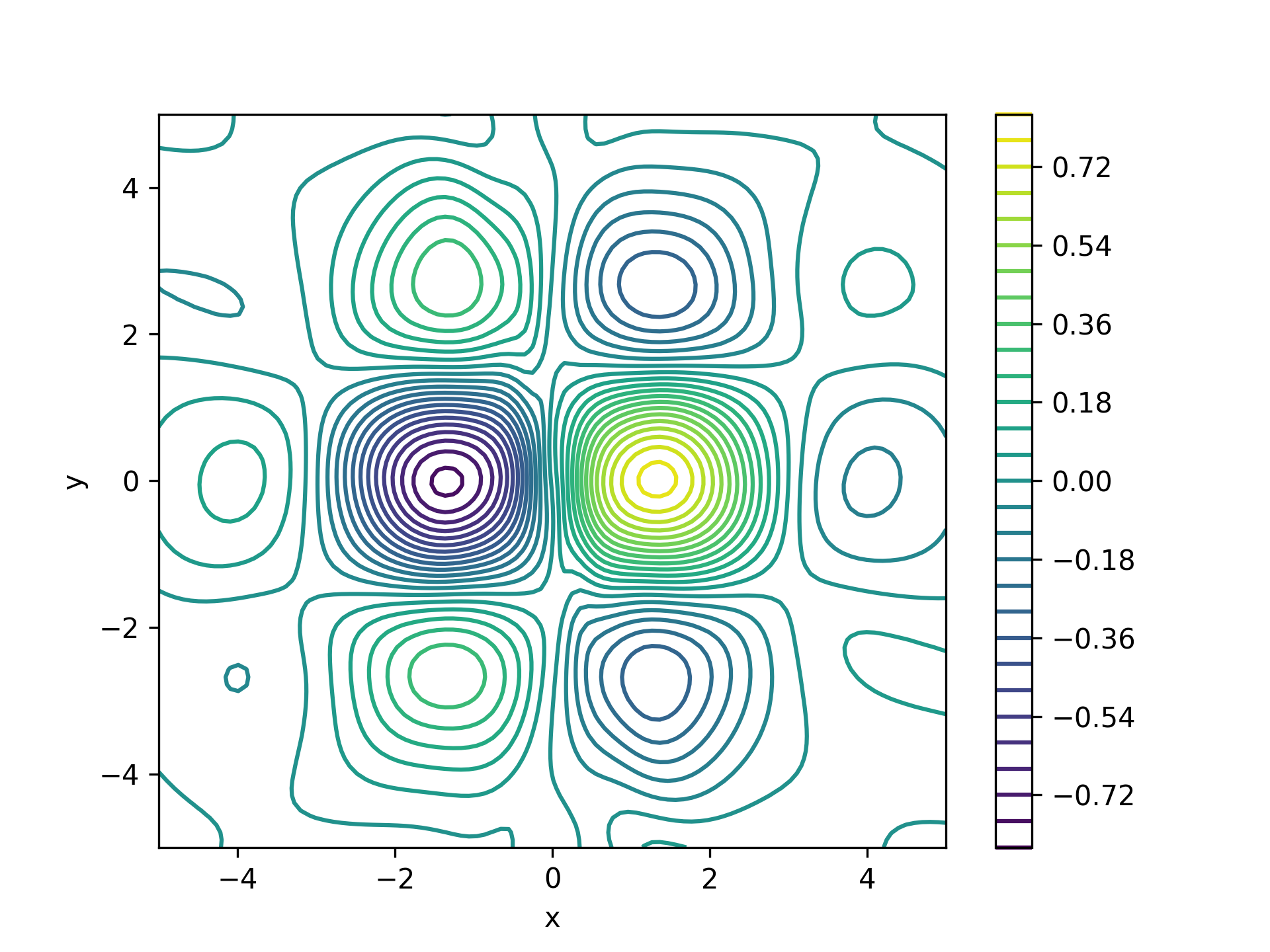}
             \caption{Random layer insertion}
    \end{subfigure}
    \begin{subfigure}[b]{0.3\textwidth}
        \includegraphics[scale=0.37]{Figures/true.png}
          \caption{True solution}
    \end{subfigure}
    \caption{Predicted contours by different approaches. Our approach outperformed all the other adaptation strategies.}
\label{comparison}
\end{figure}

\begin{table}[h!]
\caption{Best test loss (Mean squared error) achieved by different adaptation strategies along with the computational time. }
\centering
\begin{tabular}{|c | c | c |c| }
        \hline
     & Test loss achieved  & Time taken \\ 
      & (mean squared error) &   \\  \hline
    Proposed approach  & \textcolor{green!50!black}{$9.0\times 10^{-6}$ (final test loss)}& $46$ min\\
      &$2.5\times 10^{-5}$ (intermediate loss) & $14$ min\\ \hline
         Random layer insertion &$9.41\times 10^{-5}$& $5$ min  \\ \hline
Net2DeeperNet \cite{chen2015net2net}&$7.66\times 10^{-5}$ & $4.5$ min \\  \hline
Forward thinking \cite{hettinger2017forward} &$7.1\times 10^{-4}$  & $2.5$ min  \\  \hline
Baseline network &$3.82\times 10^{-5}$  & $9$ min \\  \hline
\end{tabular} 
 \label{num_result}
\end{table}

Further, Figure \ref{comparison} shows the learnt contours of the function using  different adaptation strategies. It is clear from Figure \ref{comparison} that our proposed approach provides the most accurate predictions. Further, Table \ref{num_result} shows the mean squared loss achieved on the test dataset by different approaches.

Table \ref{num_result} shows that our approach produced the best test loss in comparison to other adaptation strategies. However, we note that the training time of our approach is signifiantly higher than other approaches. The reason for this is as follows:
\begin{enumerate}
    \item Subdiscretization for the states and adjoints, i.e the use of parameter $K$ leads to slower forward propagation for our network (section \ref{our_archi}) in comparison to a conventional neural network architecture.
    \label{point_1}
    \item We note that other adaptation strategies terminated early (after adding a few layers) since no improvement in result (decrease in validation loss) was observed on adding a new layer (see line 4 in Algorithm \ref{Algo_full}).
    \label{point_2}
    \item Additionally, the optimizer used in Algorithm \ref{Algo_full} automatically stops training if no improvement in validation loss is seen for a specified number of consecutive epochs (details are provided in Appendix \ref{hyper_parameter_n}). We observed that our approach trains for a larger number of epochs, demonstrating better generalization capability compared to other approaches. This also leads to an increased training time.
     \label{point_3}
\end{enumerate}
Table \ref{num_result} also shows that one may terminate the algorithm much earlier and still achieve a reasonably low test loss of $2.5\times 10^{-5}$ in a shorter computational time ($14$ min).

\subsection{Learning the observable to parameter map for Navier Stokes equation.}
\label{navier}
In this section, we consider an inverse problem concerning the 2D Navier-Stokes equation for
viscous and incompressible fluid \cite{chorin1968numerical,canuto2006spectral} in the vorticity-streamfunction formulation written as:
\begin{equation}
\begin{aligned}
    \partial_t \omega(\bx,\ t)+\bv(\bx,\ t)\cdot \nabla \omega(\bx,\ t)&=\nu \Delta \omega(\bx,\ t),\quad \bx \in [0,\ 2\pi]^2,\ t\in (0,\ T_t] \\
    \nabla \cdot \bv(\bx,\ t)&=0,\quad \quad \quad \quad \quad \quad \quad \bx \in [0,\ 2\pi]^2,\ t\in (0,T_t]\\
    \omega(\bx,\ 0)&=\omega_0(\bx),\quad \quad \quad \quad \quad \quad \bx \in [0,\ 2\pi]^2,
\end{aligned}
\label{navier_equation_o}
\end{equation}
where $\bv(\bx,\ t)$ is the velocity field, $\omega(\bx,\ t)$ is the vorticity, $\omega_0(\bx)$ is the initial vorticity, and $\nu=10^{-3}$ is the viscosity coefficient. Periodic boundary conditions are imposed on all boundaries. The spatial domain is discretized with $64 \times 64$ uniform mesh, and we consider the
time horizon $t \in (0, 0.5)$. Our objective is to reconstruct the initial vorticity $\omega_0(\bx)$ from the measurements of vorticity at $10$ randomly selected observation points at the final time $T_t=0.5$ (see right subfigure in Figure \ref{fig:nav_mesh}).  We therefore construct the input vector for the neural network as $[\omega(\bx_1,\ 0.5),\ \dots, \omega(\bx_{n_0},\ 0.5)]$, where $\bx_i$ are fixed locations  on the domain, and $n_0=10$. The output of the neural network are coefficients $\bc = (c_1,\hdots, c_{n_{T+1}})$ which can then be used to reconstruct the initial vorticity $\omega_0(\bx)$ (see \eqref{kl_nav} below). We choose $n_{T+1}=50$.

{\it{ Note that the combination of information loss through temporal evolution and sparse spatial measurements makes this inverse problem severely ill-posed.}}

\begin{figure}[htbp]
    \centering

        \includegraphics[scale=0.4]{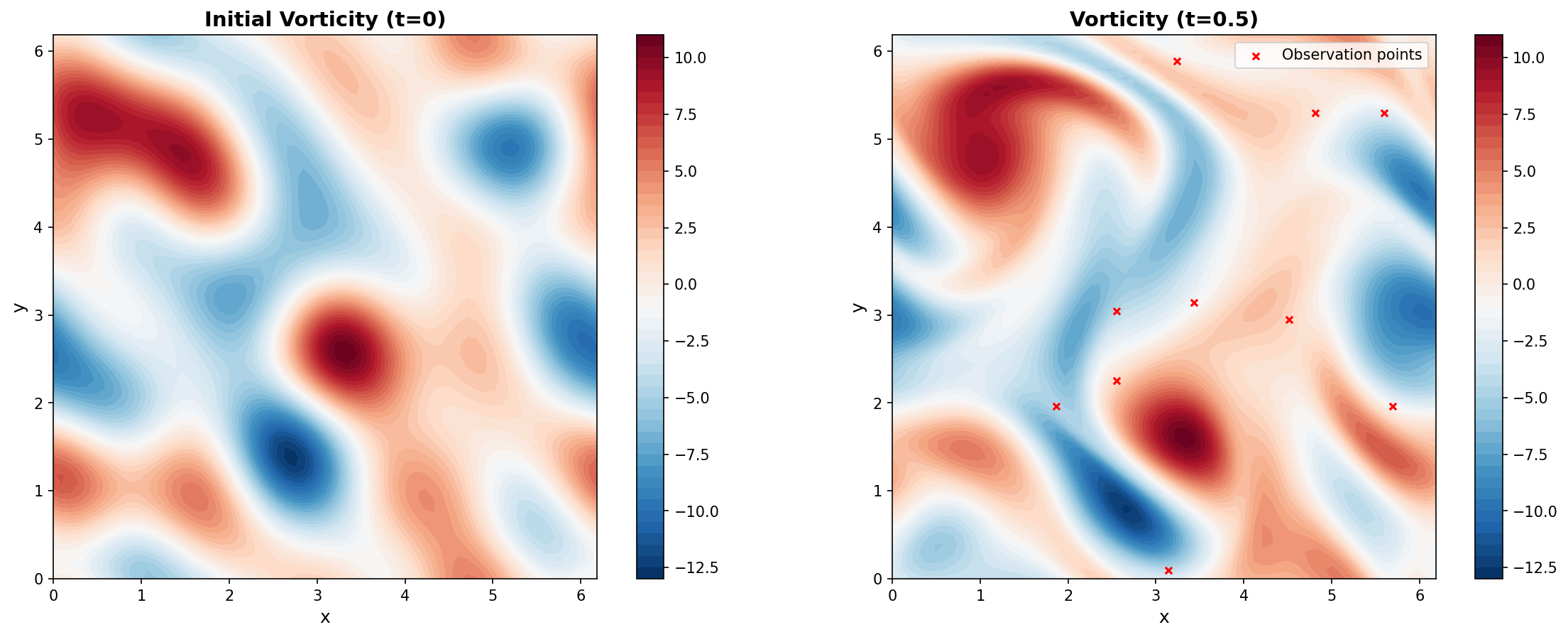}
        \caption{Left to right: Initial vorticity field; Vorticity field at time $t=0.5$ along with observation points (marked in red) used for inversion.}
    \label{fig:nav_mesh}
\end{figure}

\vspace{0.2cm}
\noindent{\bf{\underline{Data generation and numerical results}}}
For learning the inverse map, we draw
samples of $\omega(\bx,\ 0)$ based on the truncated Karhunen-Lo\`eve expansion as follows:
\begin{equation}
    \omega(x,\ 0)=\sum_{i=1}^{n_{T+1}}\sqrt{\lambda_i}\phi_i(x)c_i,
    \label{kl_nav}
\end{equation}
where $\bc = (c_1,\hdots, c_{n_{T+1}})$ are the KL coefficients drawn uniformly from $[0, 1]$, and $(\lambda_i, \ \phi_i)$ are eigenpairs obtained by the eigendecomposition of the squared exponential covariance kernel
\begin{equation}
K(\bx, \bx') = \exp\left(-\frac{\|\bx - \bx'\|^2}{2\ell^2}\right),
\end{equation}
with length scale $\ell = 0.3$, subject to periodic boundary conditions. For demonstration, we choose $n_{T+1}=50$. For a given $\omega_0(\bx)$, we solve the Navier-Stokes equation \eqref{navier_equation_o} using a pseudospectral method in Fourier space with the RK45 (Runge-Kutta) time integration scheme  to compute $\omega(\bx,\ 0.5)$ on the grid, which is then used to generate the observation vector $[\omega(\bx_1,\ 0.5),\ \dots, \omega(\bx_{n_0},\ 0.5)]$ at the fixed observation locations. For a new observation data, the network outputs the vector $\bc$ which can then be used to reconstruct $\omega_0(\bx)$ using \eqref{kl_nav}. In addition, $1 \%$ additive Gaussian noise is added to the observations $\by$ to represent the actual field condition.

We consider experiments with training data set of size $S=700$. 
We consider an additional $100$ data points for validation data set and $300$ data points for testing data set. Other details on the hyperparameter settings are provided in Table \ref{hyper_parameter_n}. 

Figure \ref{decomp_nav} shows the decomposition of error $\sE_n$ computed using \eqref{error_est} for the first few iterations of Algorithm \ref{Algo_full}. We clearly see that upon adding a new hidden layer at the location of maximum error (indicated as a red block in Figure \ref{decomp}) and retraining the network, the error decreases at the location. 
\begin{figure}[h!]
    \centering
    \begin{subfigure}[b]{0.45\textwidth}
        \centering
        \includegraphics[width=\textwidth]{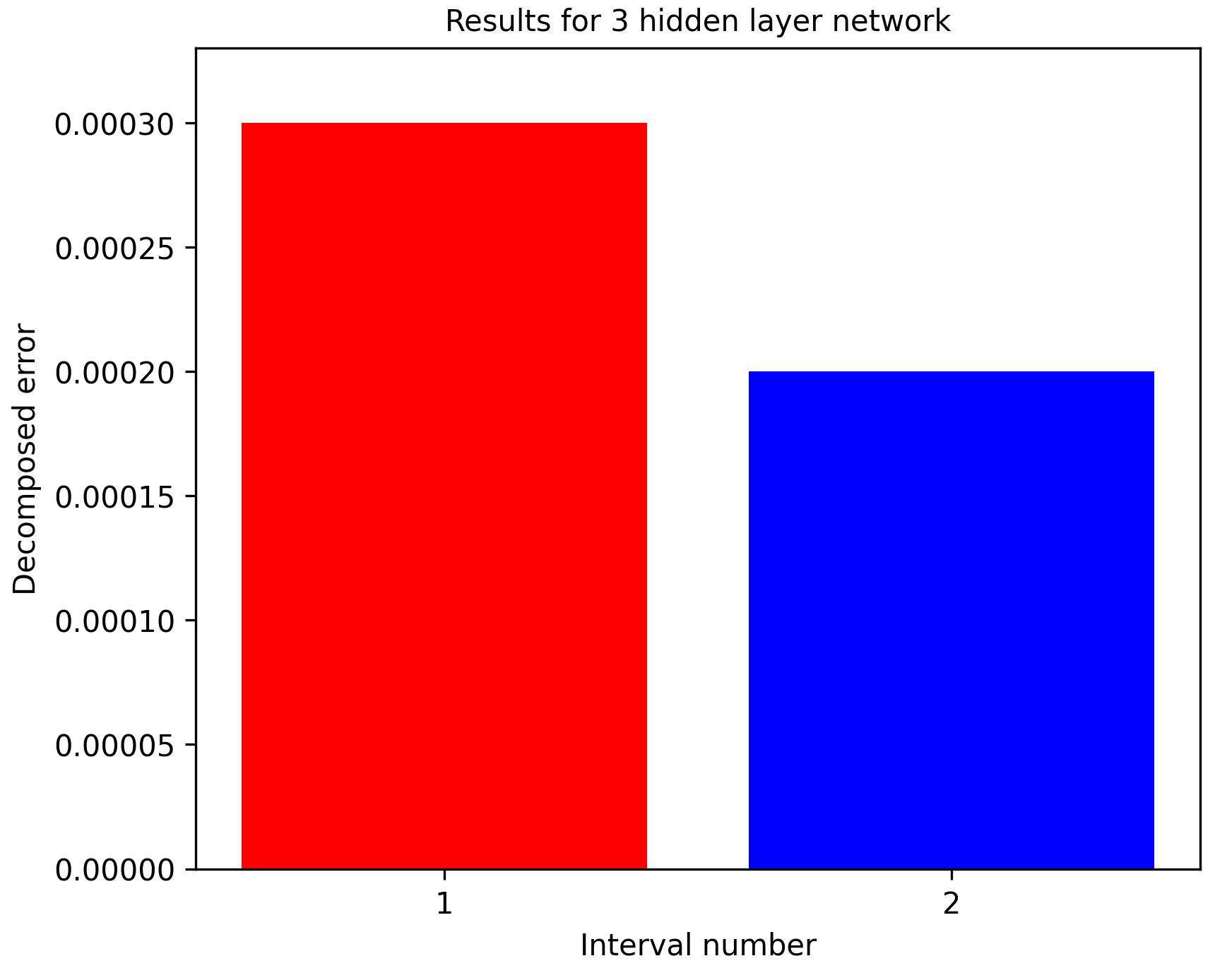}
        \caption{First iteration}
    \end{subfigure}
    \hfill
    \begin{subfigure}[b]{0.45\textwidth}
        \centering
        \includegraphics[width=\textwidth]{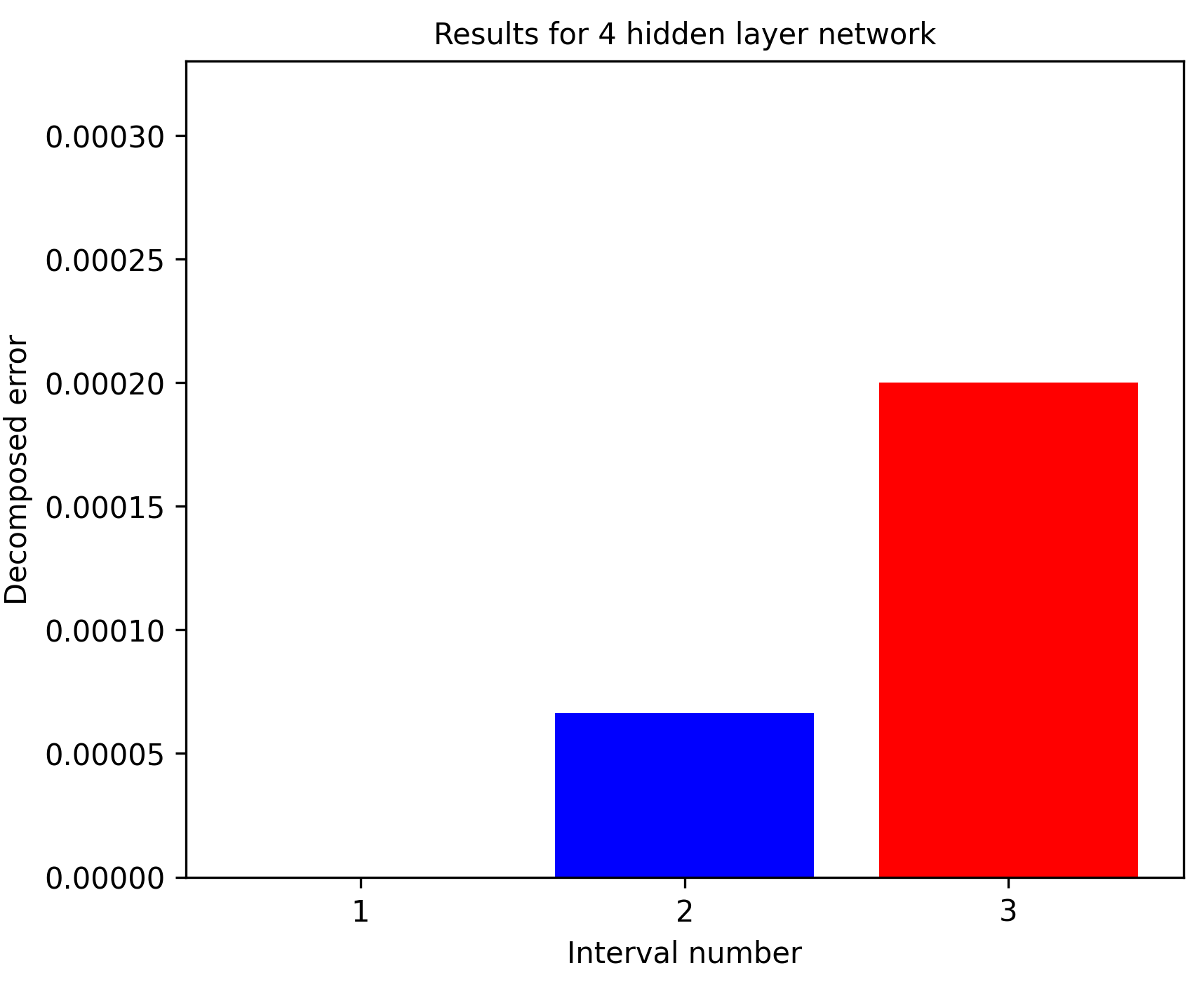}
        \caption{Second iteration}
    \end{subfigure}
    
    \vspace{0.5cm}
    
    \begin{subfigure}[b]{0.45\textwidth}
        \centering
        \includegraphics[width=\textwidth]{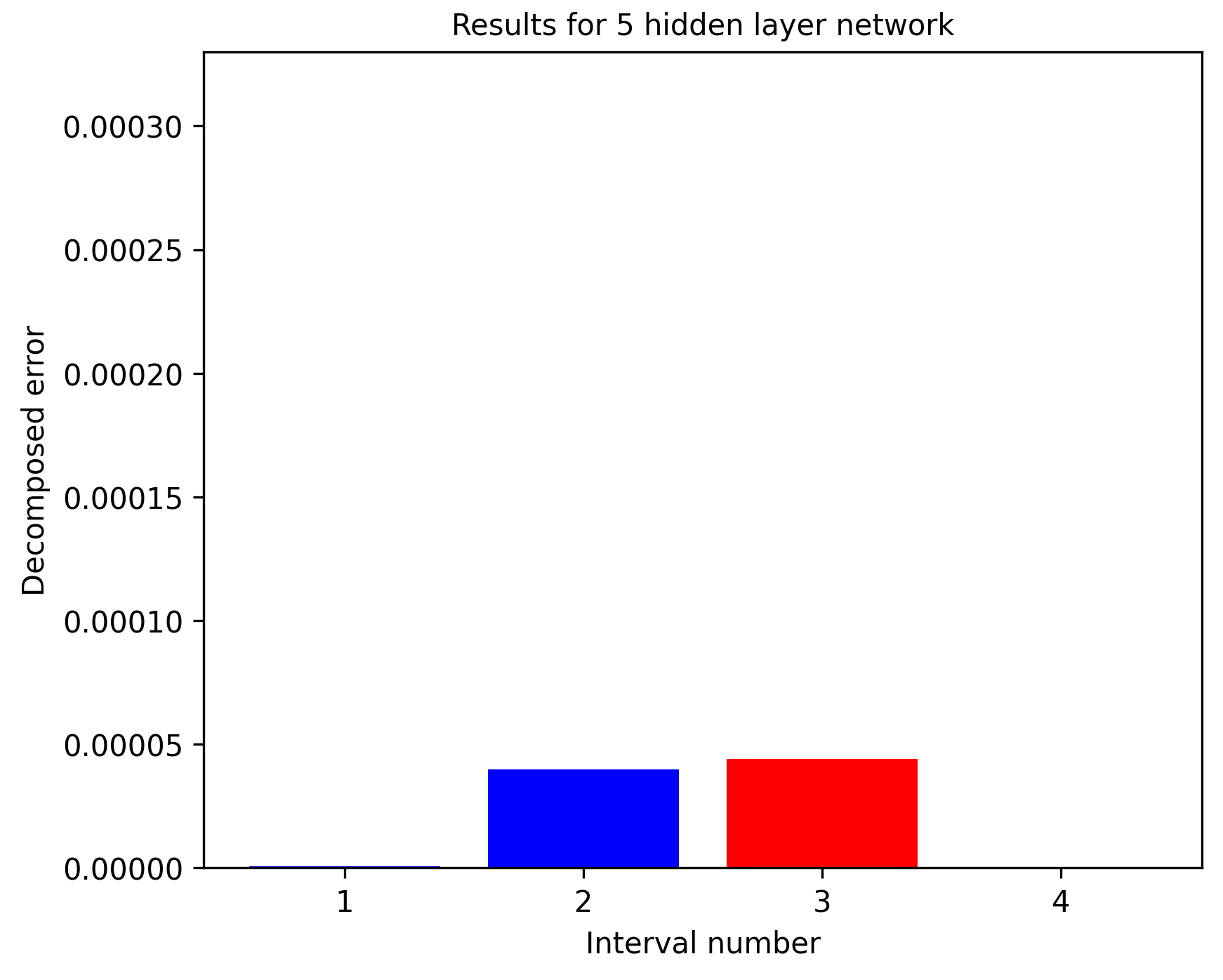}
        \caption{Third iteration}
    \end{subfigure}
    \hfill
    \begin{subfigure}[b]{0.45\textwidth}
        \centering
        \includegraphics[width=\textwidth]{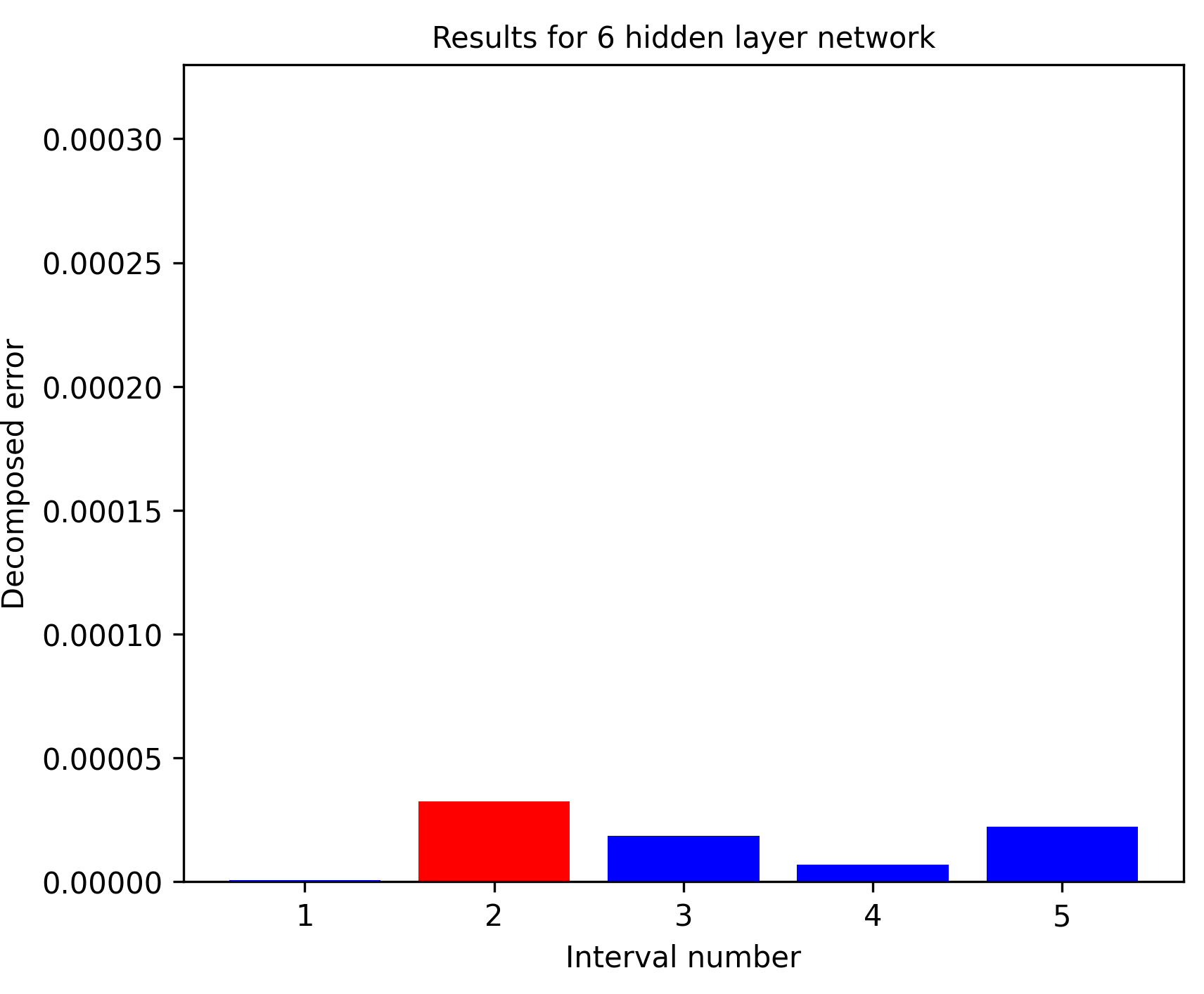}
        \caption{Fourth iteration}
    \end{subfigure}
    \caption{Decomposition of the a-posteriori error $\sE_n$ in \eqref{error_est} for the first few iterations of the algorithm. Block in red denotes the interval of maximum error where a new layer is added for the next training phase.}
    \label{decomp_nav}
\end{figure}

To quantify the performance of our algorithm, we compute the average
relative errors on the test dataset  as follows:
\begin{equation}
    \text{Err}=\frac{1}{M}\sum_{i=1}^M \frac{\norm{\boldsymbol{\omega}^{pred}_i-\boldsymbol{\omega}^{true}_i}^2}{\norm{\boldsymbol{\omega}^{true}_i}^2},
    \label{avg_rel_err}
\end{equation}
where $M$ denotes the number of test data samples, $\boldsymbol{\omega}^{pred}_i$ 
denotes the neural network prediction for the  $i^{th}$ test sample. Note that the network outputs the vector $\bc$ which is then used to reconstruct $\omega_0(\bx)$ using \eqref{kl_nav}. Here, $\boldsymbol{\omega}^{\text{pred}}_i$ denotes the vector of solution $\omega_0(\bx)$ on a $64 \times 64$ grid. $\boldsymbol{\omega}^{true}_i$ denotes the corresponding synthetic ground truth
vorticity field. 
\begin{figure}[h!]
    \centering
             \includegraphics[scale=0.25]{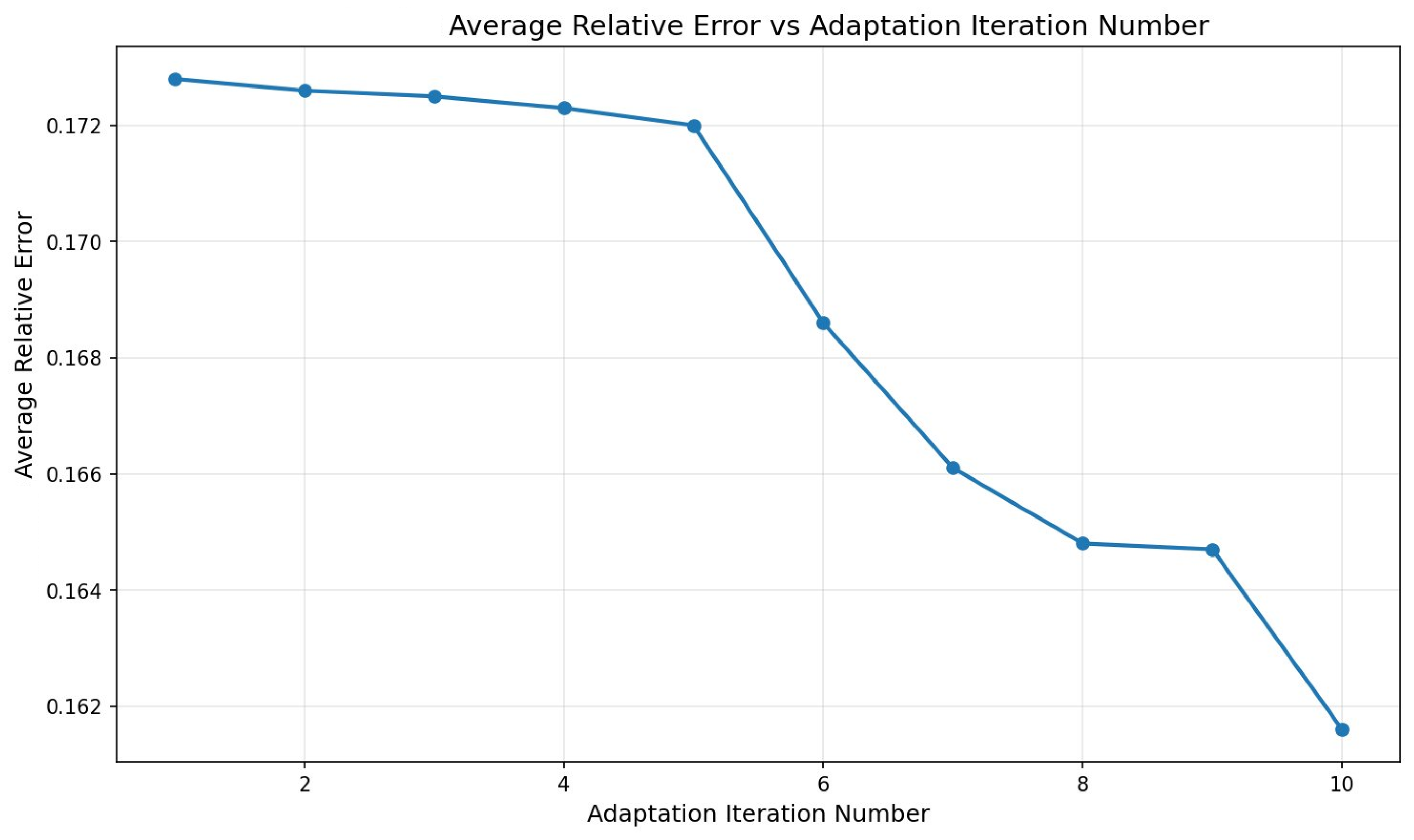}
    \caption{Average relative error \eqref{avg_rel_err}  achieved on the test dataset at different iterations of the algorithm.}
    \label{loss_nav}
\end{figure}
Figure \ref{loss_nav} shows how the average relative error on the test dataset decreases at different iterations of the algorithm. As more layers are added, the network becomes more expressive, leading to improved generalization. In addition, Figure \ref{pred_navier_sam} shows the vorticity field (a random test sample) predicted by our proposed approach at termination of our Algorithm \ref{Algo_full}. Figure \ref{pred_navier_sam} shows that the predicted vorticity is close to the true vorticity field.
\begin{figure}[h!]
    \centering
    \begin{subfigure}[b]{0.4\textwidth}
        \centering
        \includegraphics[width=\textwidth]{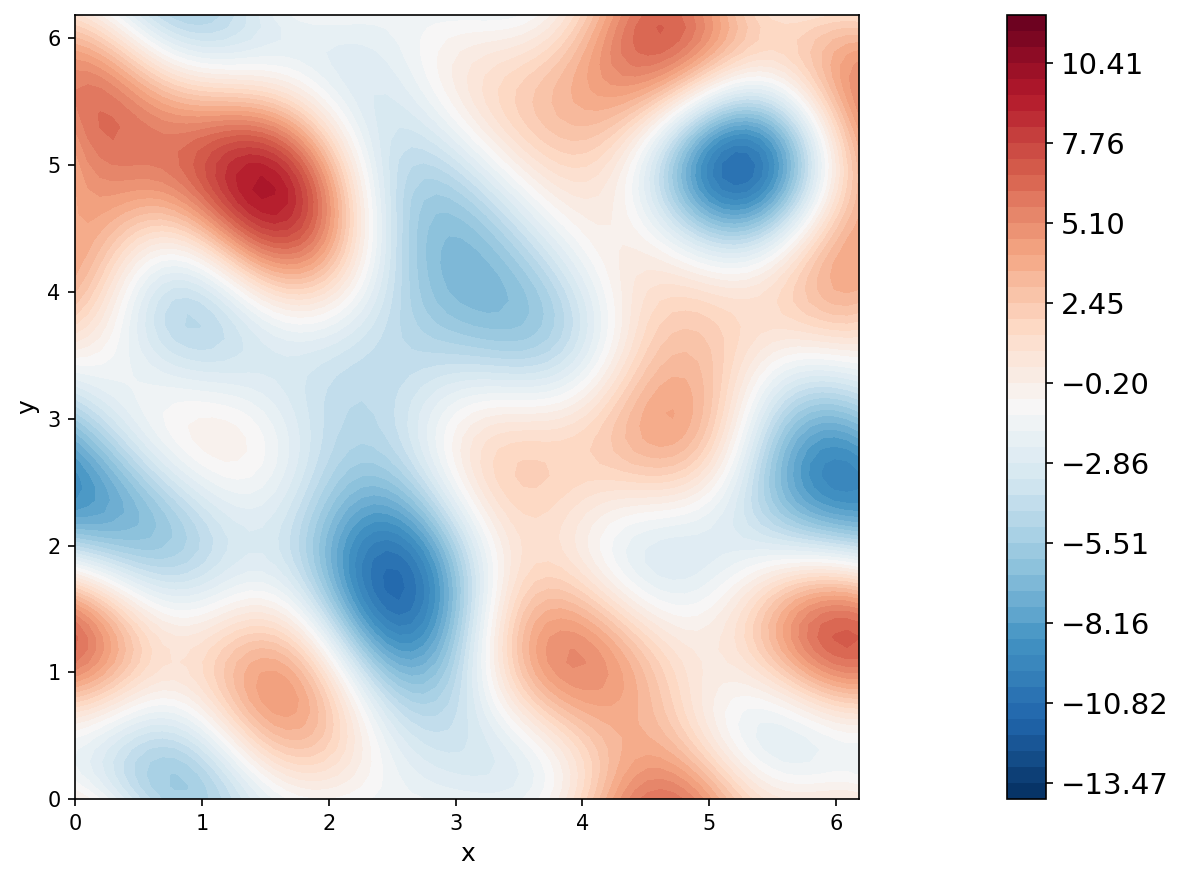}
    \end{subfigure}
    \hspace{1 cm}
    \begin{subfigure}[b]{0.4\textwidth}
        \centering
        \includegraphics[width=\textwidth]{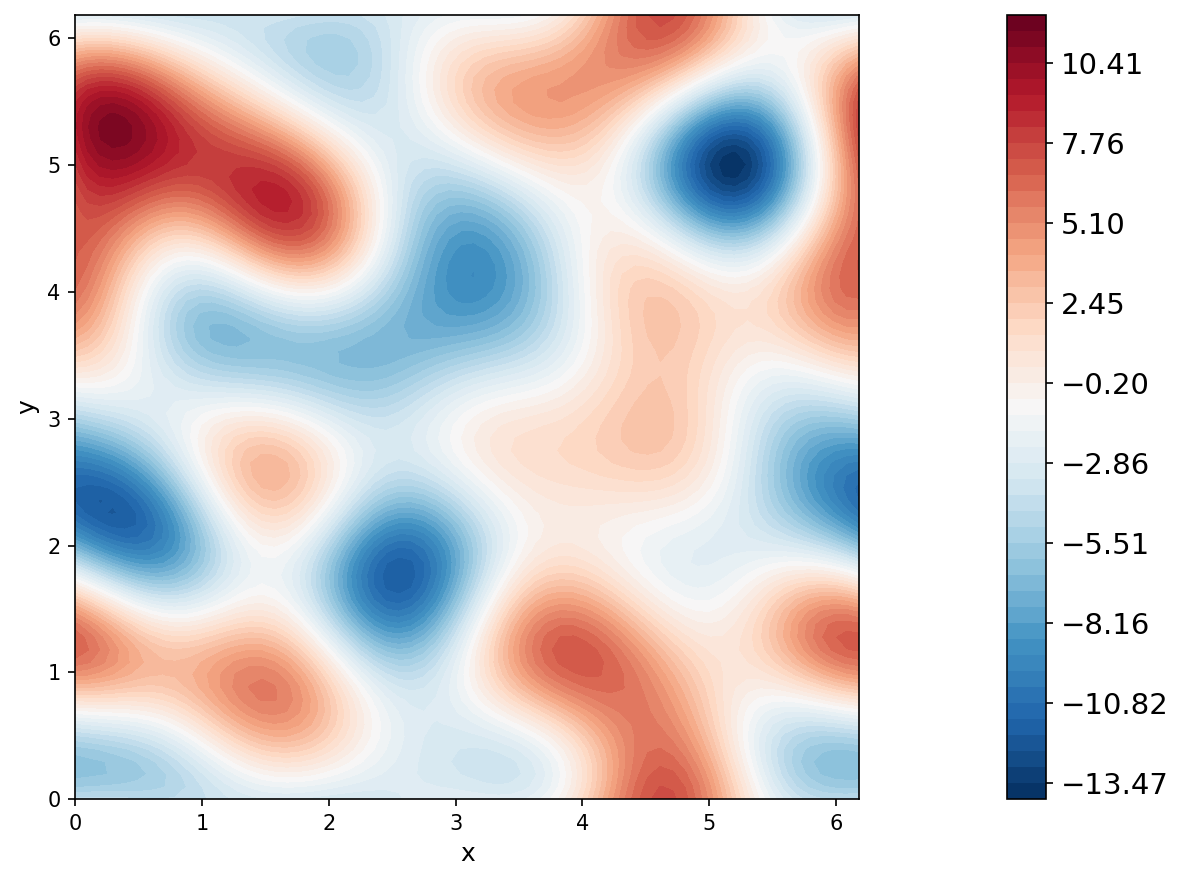}
    \end{subfigure}
    
    \caption{Left to Right: Predicted vorticity field; and true vorticity field for a randomly chosen test sample.}
    \label{pred_navier_sam}
\end{figure}
 Further, Table \ref{num_result_2} shows the average relative error \eqref{avg_rel_err} achieved on the test dataset by different approaches. Table \ref{num_result_2} shows that our approach produced the best average relative error in comparison to other adaptation strategies. However, note that the training time of our approach is  higher than other approaches and the reason for this is discussed in item \ref{point_1}, item \ref{point_2}, and item \ref{point_3}.
\begin{table}[h!]
\caption{Best test loss (average relative error) achieved by different adaptation strategies along with the computational time. }
\centering
\begin{tabular}{|c | c | c |c| }
        \hline
     & Test loss achieved  & Time taken \\ 
      & (average relative error) &   \\  \hline
    Proposed approach  & \textcolor{green!50!black}{$0.161$ (final test loss)}& $9$ min\\
      &$0.165$ (intermediate loss) & $5$ min\\ \hline
         Random layer insertion &$0.170$& $1$ min  \\ \hline
Net2DeeperNet \cite{chen2015net2net}&$0.171$ & $2$ min \\  \hline
Forward thinking \cite{hettinger2017forward} &$0.172$  & $30$ sec  \\  \hline
Baseline network &$0.166$  & $3$ min \\  \hline
\end{tabular} 
 \label{num_result_2}
\end{table}
Further, in an attempt to estimate the accuracy of each adaptation strategy (for a graphical representation of the error by each approach), we define the pointwise average relative error as:
\begin{equation}
    \text{Err}_j=\frac{1}{M}\sum_{i=1}^M \frac{\LRp{\boldsymbol{\omega}^{pred}_{i,j}-\boldsymbol{\omega}^{true}_{i,j}}^2}{\norm{\boldsymbol{\omega}^{true}_i}^2/|\boldsymbol{\omega}_i|},
    \label{error_metric_nav}
\end{equation}
where subscript $j$ denotes the $j^{th}$ component of $\boldsymbol{\omega}_i$ and $|\boldsymbol{\omega}_i|$ denotes the number of elements in the vector. Figure \ref{pt_navier} shows the error $\text{Err}_j$ plotted for  the main adaptation strategies. It is quite clear from  Figure \ref{pt_navier} that our approach outperformed all other adaptation strategies in terms of producing a lower error over the spatial domain. 
\begin{figure}[h!]
    \hspace{0.7 cm}
    \begin{subfigure}[b]{0.3\textwidth}
        \includegraphics[scale=0.25]{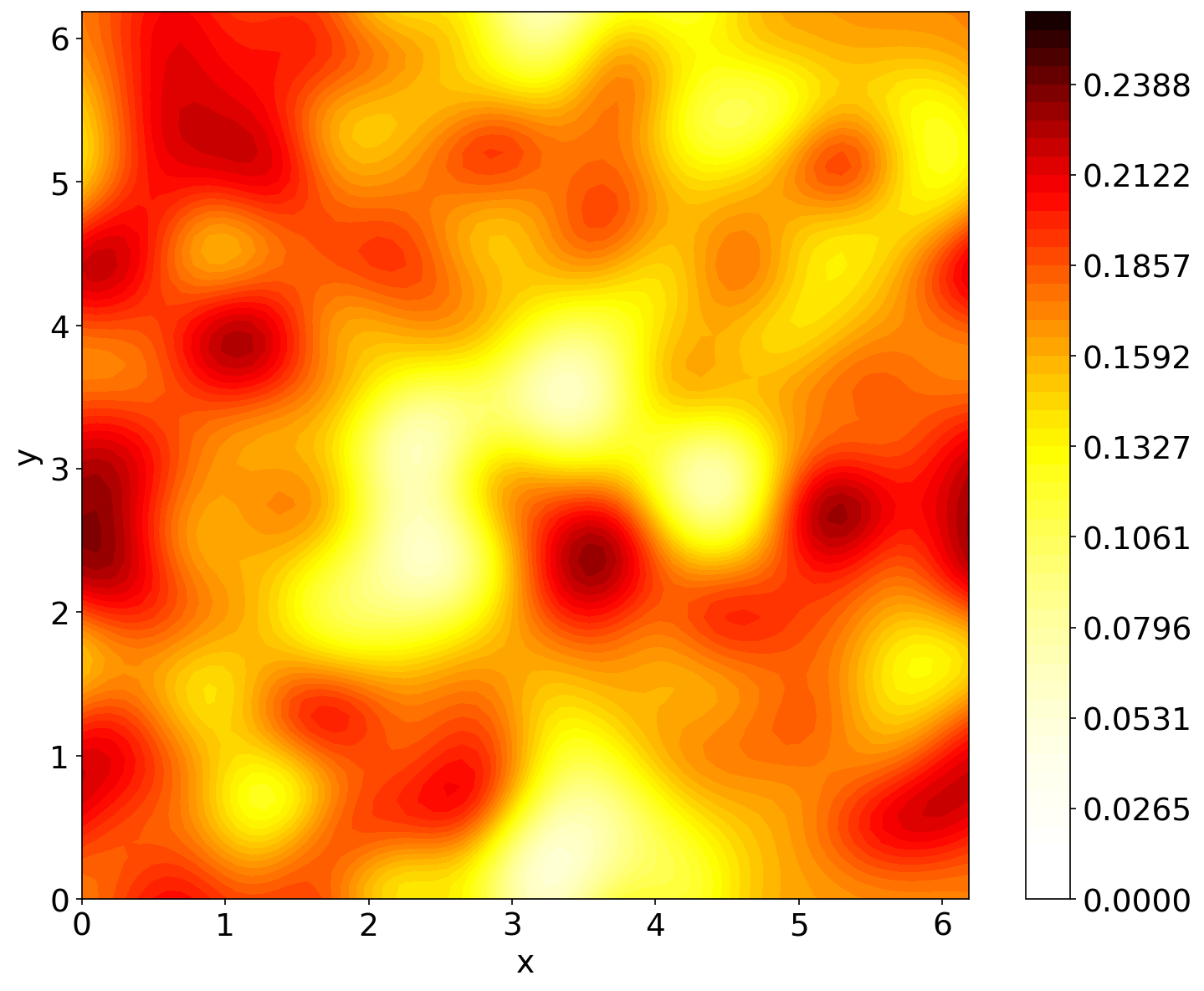}
        \caption{Proposed approach}
    \end{subfigure}
    \hspace{2 cm}
    \begin{subfigure}[b]{0.3\textwidth}
        \centering
        \includegraphics[scale=0.25]{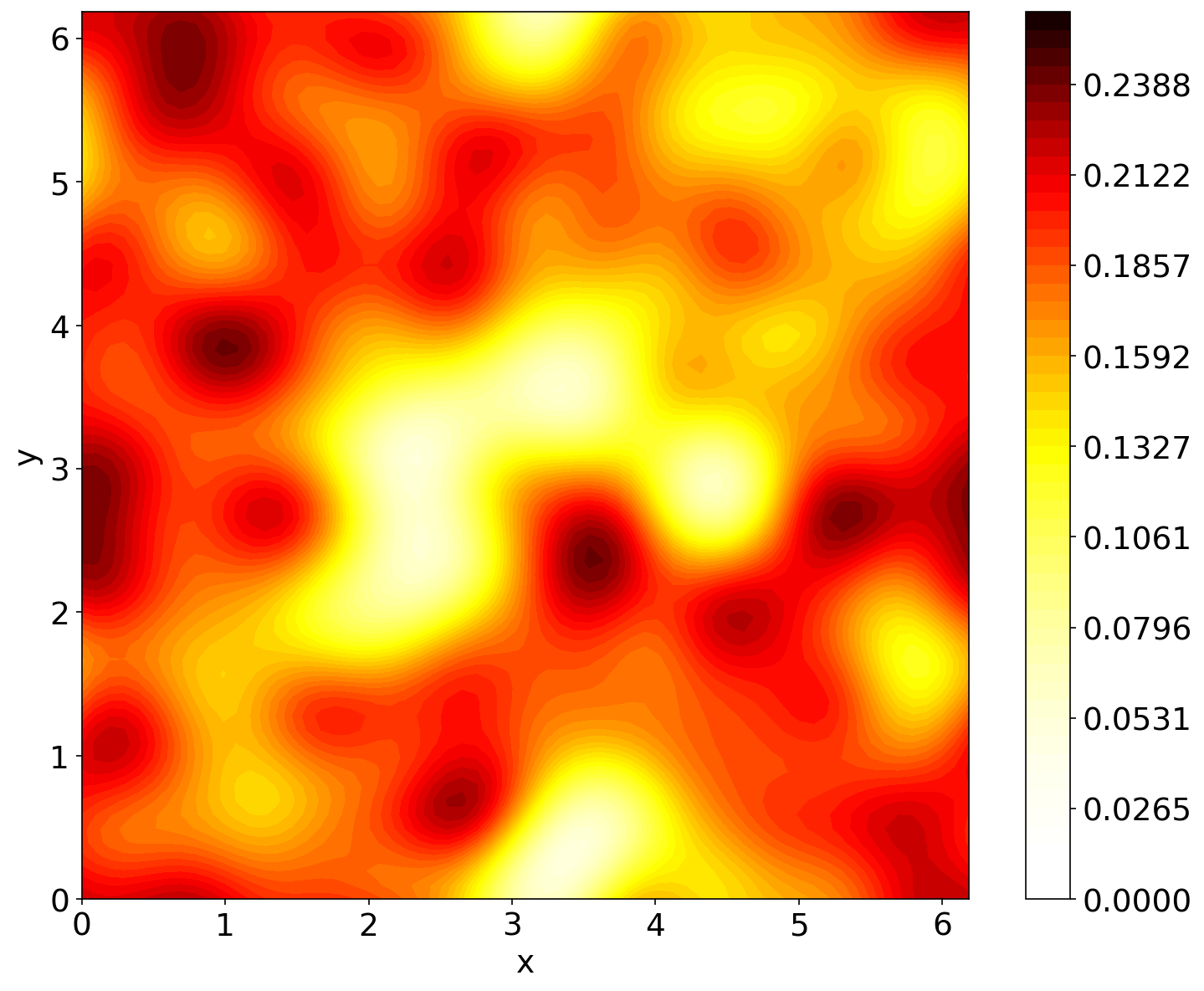}
        \caption{Random layer insertion}
    \end{subfigure}
    
    \vspace{0.5cm}
    \hspace{0.7 cm}
    \begin{subfigure}[b]{0.3\textwidth}
        \centering
        \includegraphics[scale=0.25]{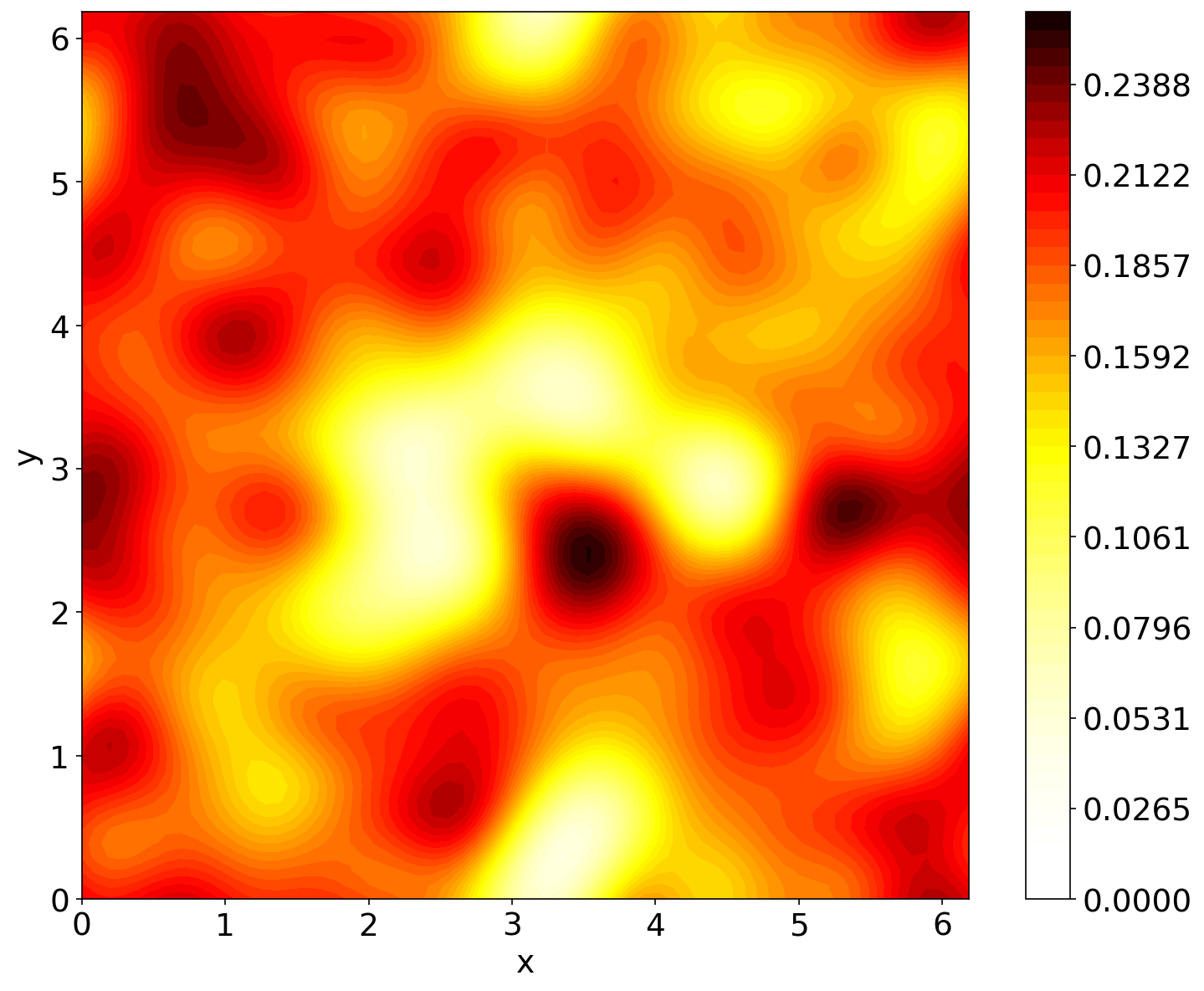}
        \caption{Net2DeeperNet \cite{chen2015net2net}}
    \end{subfigure}
     \hspace{2cm}
    \begin{subfigure}[b]{0.3\textwidth}
        \centering
        \includegraphics[scale=0.25]{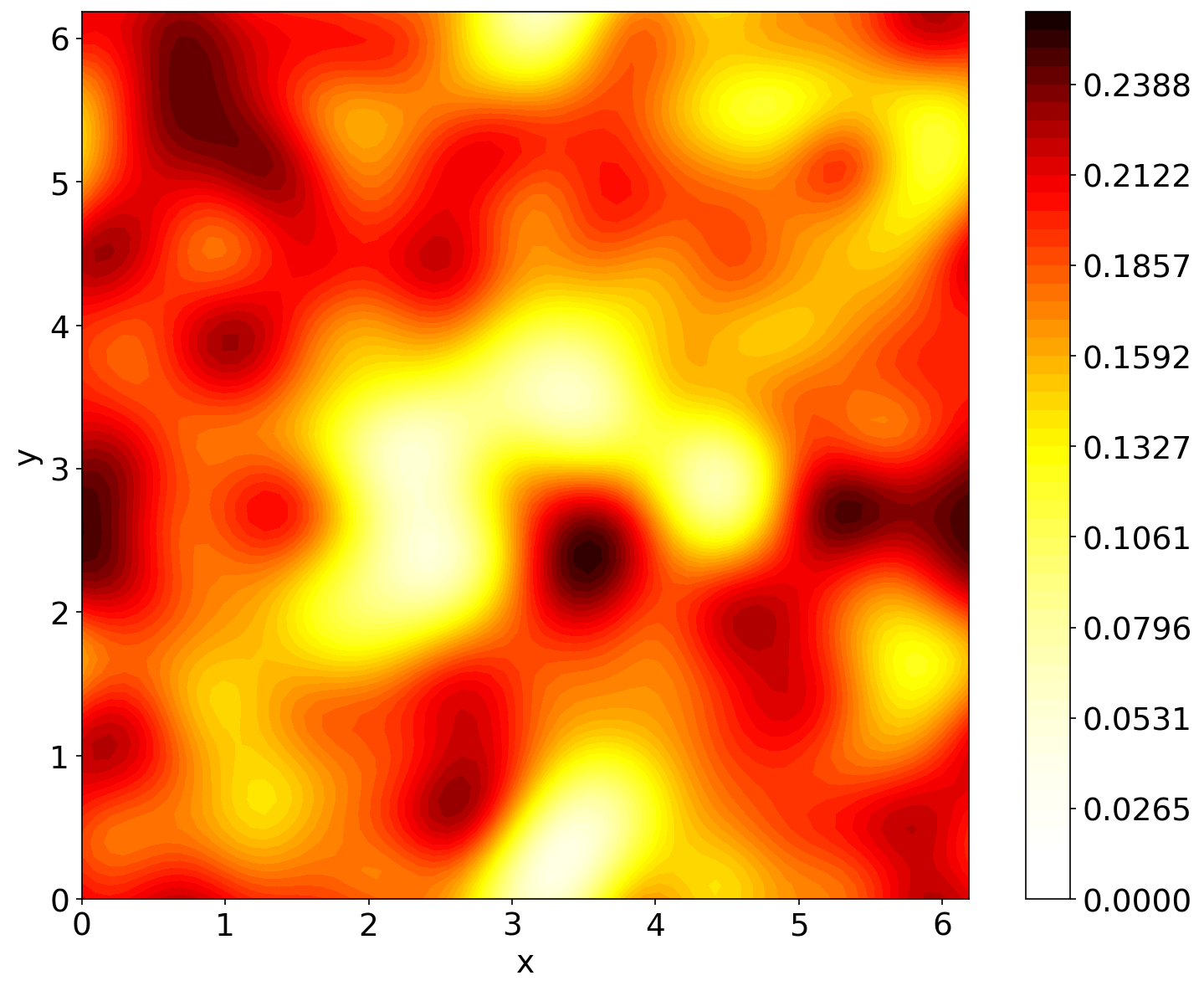}
        \caption{Forward thinking \cite{hettinger2017forward}}
    \end{subfigure}
    
    \vspace{0.5cm}
    \hspace{4.5cm}
    \begin{subfigure}[b]{0.3\textwidth}
        \centering
        \includegraphics[scale=0.25]{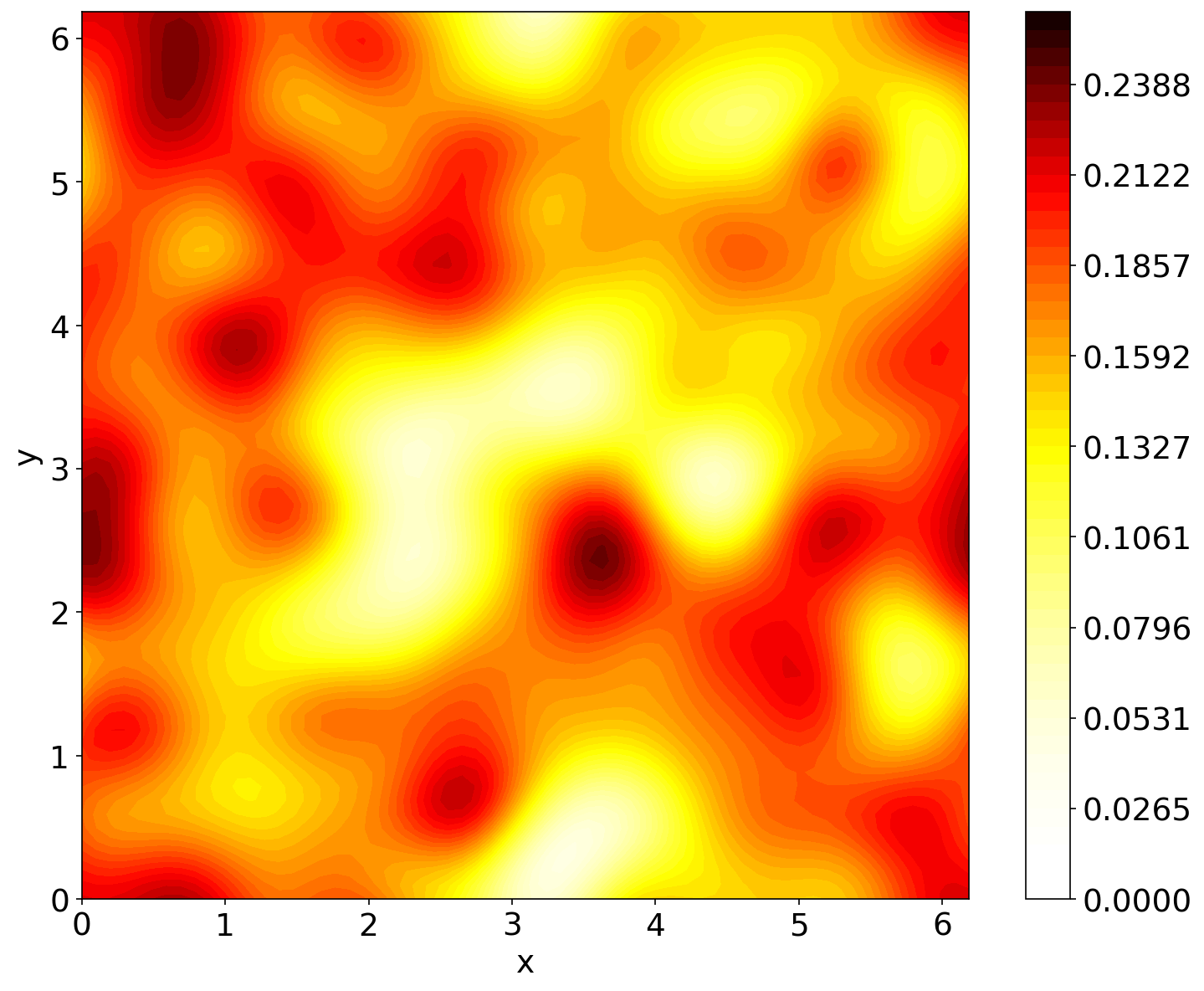}
        \caption{Baseline network}
    \end{subfigure}
    
    \caption{Pointwise average relative error \eqref{error_metric_nav} achieved by different approaches on the spatial domain. Our approach outperformed all the other adaptation strategies.}
    \label{pt_navier}
\end{figure}

\section{Conclusion}

This work presents a rigorous mathematical framework for neural network architecture adaptation based on a posteriori error estimation. By formulating neural network training as a continuous-time optimal control problem and discretizing it using finite element method, we derived computable error bounds that characterize how approximation error distributes across network layers. This error decomposition provides a principled mechanism for depth adaptation: layers are inserted at locations of maximum error, enabling efficient refinement of the network architecture.

The key contributions of our approach are threefold. First, we established theoretical error estimates (Corollary \ref{cor_one}) that rigorously bound the functional error arising from piecewise linear representation of network parameters. Second, we developed a practical computational framework (Algorithm \ref{Algo_full}) that implements these theoretical insights through a two-level discretization scheme, balancing theoretical rigor with computational efficiency. Third, we demonstrated through numerical experiments that our method outperforms existing architecture adaptation strategies in the small network width regime.

One drawback of the approach is the increased computational cost compared to other approaches due to the subdiscretization parameter $K$ involved (see  section \ref{discretize_state}). However, this investment yields substantial improvements in final model quality as evident from our numerical results. For applications where model accuracy is paramount, the additional training time is justified by superior performance. 

Our future work includes extending the framework to other architectures such as convolutional and recurrent neural networks, and incorporating width adaptation alongside depth adaptation.

\appendix

\section{General setting for numerical experiments}{}
   \label{hyper_parameter}
All codes were written in PyTorch. Throughout the study, we have employed the Adam optimizer for minimizing the loss function.      Our proposed approach is compared with a number of different approaches as given below: 
\begin{equation*}
 \begin{aligned}
\text{Proposed approach}&: \text{Architecture adaptation algorithm as described }\\
& \text{\ \ \ in Algorithm \ref{Algo_full}}.\\
\text{Random layer insertion}&: \text{Architecture adaptation algorithm as described in}\\
& \text{\ \ \  Algorithm \ref{Algo_full} with $n^*=\arg \min_n\{ \sE_n\}$ in line 6 of  Algorithm \ref{Algo_full}}.\\
\text{ Net2DeeperNet  }&: \text{Increasing depth of network based on  based on function }\\
& \text{\ \ \  preserving transformations. A layer is inserted at random  }\\
& \text{\ \ \ position with a small Gaussian noise added to the }\\
& \text{\ \ \ parameters to break symmetry \cite{chen2015net2net}}.\\
\text{ Baseline network (B)}&: \text{Training a randomly initialized network with the same }\\
& \text{\ \ \ final layers $T$ as obtained by our proposed approach}.\\
\text{Forward Thinking (H)}&: \text{Algorithm for layerwise adaptation proposed by}\\
& \text{\ \ \  Hettinger et al. \cite{hettinger2017forward}}.\\
\end{aligned}  
\end{equation*}
We maintain the same activation functions and hyperparameters for all the adaptation strategies in order to make a fair comparison. For ``proposed approach" and ``random layer insertion", we use the architecture described in section \ref{our_archi}. For all other approaches we use a conventional feed-forward neural network.
Note that, for all methods, the reported numerical results correspond to the model that achieved the best validation loss.  
\vspace{0.1 cm}

\noindent{\bf{\underline{Activation function employed}}}

\vspace{0.1 cm}

For all approaches, we use the {\it{tanh}} activation function in the input and hidden layers, with a {\it{linear}} activation in the output layer.

\vspace{0.1 cm}

\noindent{\bf{\underline{Random search for best initialization of parameters}}}

\vspace{0.1 cm}

Note that each approach mentioned above requires random initialization of the network parameters to start the adaptation procedure. In this work, we consider 20 random initializations of the initial small network and retain the best performing network (lowest validation loss) for subsequent adaptation.

\section{Details of hyperparameter values for different problems}{}
   \label{hyper_parameter_n}
   Details of hyperparameters used in  Algorithm \ref{Algo_full}  is provided in  Table \ref{Input_values_different_problems}.   Table \ref{Input_values_different_problems} additionally provides details on the hyperparameters used for the optimizer. 
The description of each problem is also provided below.
\begin{table}[h!]
	\caption{Details of hyperparameters  for Algorithm \ref{Algo_full}} \label{Input_values_different_problems}
	\renewcommand{\arraystretch}{1.3}
	\centering	
	\begin{tabular}{|c | c | c | c | c | c | c | c | c |c|c|c|}
 \hline
		\centering
		 Problem &$n_1$ &$K$ &$N_n$ &   $ T$ & $t_1$ & $t_T$ & $E_e$ & $b_s$ & $\ell_r$&  $\tau$\\
		\hline
   I & 5 &4 & 15& 3 & 0 &1& 1000 &100& 0.01&  200 \\ \hline
   II & 20 &4 & 10& 3 & 0 &1 &1000 &700& 0.001& 200\\ \hline
	\end{tabular} 
\end{table}
In Table \ref{Input_values_different_problems}, $b_s$ denotes the batch size, $\ell_r$ denotes the learning rate, $E_e$ denotes the maximum number of epochs for which the network is trained at each iteration of Algorithm \ref{Algo_full}. $\tau$ denotes the patience parameter, specifying the number of consecutive epochs without improvement in validation loss after which training is automatically terminated. For initializing the weights and biases,  we choose  a zero mean Gaussian noise  with standard deviation $\sigma_n=0.01$.
The description of different problems in Table \ref{Input_values_different_problems} are provided below:
\begin{equation*}
 \begin{aligned}
\text{I}&: \text{Proof of concept example in section \ref{proof_of}}.\\
\text{II }&: \text{Learning the observable to parameter map for Navier Stokes equation in section \ref{navier}}.
\end{aligned}  
\end{equation*}

\bibliographystyle{unsrt}  
\bibliography{references}

\end{document}

%% file: Macros.tex
\newcommand{\bs}[1]{\boldsymbol{#1}}

\newcommand{\LRs}[1]{\left[ #1 \right]}

\newcommand{\bx}{\boldsymbol{x}}
\newcommand{\bff}{\boldsymbol{f}}
\newcommand{\be}{\boldsymbol{e}}
\newcommand{\by}{\boldsymbol{y}}
\newcommand{\bv}{\boldsymbol{v}}

\newcommand{\bz}{\boldsymbol{z}}

\newcommand{\nor}[1]{\left\| #1 \right\|}

\newcommand{\snor}[1]{\left| #1 \right|}
\newcommand{\norm}[1]{\left \Vert #1 \right \Vert}

\def\sJ{{\mathcal J}}

\def\sUw{{\mathcal U_w}}
\def\sUb{{\mathcal U_b}}

\renewcommand{\u}{u}
\newcommand{\xb}{\bs{x}}

\newcommand{\ub}{{\bs{\u}}}

\newcommand{\Lo}{\mathcal{L}}

\newcommand{\sE}{\mathcal{E}}

\newcommand{\LRp}[1]{\left( #1 \right)} 

\newcommand{\real}{\mathbb{R}}
 
\def\bW{\boldsymbol{W}}

\def\bb{\boldsymbol{b}} 
\def\bc{\boldsymbol{c}}
\algblock{ParFor}{EndParFor}
\algnewcommand\algorithmicparfor{\textbf{For all}}
\algnewcommand\algorithmicpardo{\textbf{do in parallel}}
\algnewcommand\algorithmicendparfor{\textbf{end\ For all}}
\algrenewtext{ParFor}[1]{\algorithmicparfor\ #1\ \algorithmicpardo}
\algrenewtext{EndParFor}{\algorithmicendparfor}
